\documentclass[twoside]{article}
\usepackage[accepted]{aistats2026}
\usepackage[round]{natbib}

\usepackage{bm}
\usepackage{amsmath, amsthm, amssymb}

\newtheorem{proposition}{Proposition}
\usepackage{tikz}
\usetikzlibrary{arrows.meta, decorations.markings}
\usepackage{pgfplots}
\usepackage[dvipsnames]{xcolor}
\usepackage{pdfpages}
\usepackage{subcaption}
\usepackage{float}
\usepackage{booktabs} 
\usepackage{caption}  
\usepackage{url}
\usepackage{enumitem}

\newcommand{\Input}{\mathrm{Input}}

\usepackage{pgfplots}
\pgfplotsset{compat=1.18}

\begin{document}

%
\runningtitle{Lag Operator SSMs: A Geometric Framework for Structured State Space Modeling}

%

\twocolumn[

\aistatstitle{
Lag Operator SSMs: A Geometric Framework \\
for Structured State Space Modeling
}

\aistatsauthor{ Sutashu Tomonaga \And Kenji Doya \And  Noboru Murata }

\aistatsaddress{ Okinawa Institute of \\Science and Technology
\And  Okinawa Institute of \\Science and Technology \And Waseda University } ]

\begin{abstract}
Structured State Space Models (SSMs), which are at the heart of the recently popular Mamba architecture, are powerful tools for sequence modeling. However, their theoretical foundation relies on a complex, multi-stage process of continuous-time modeling and subsequent discretization, which can obscure intuition. We introduce a direct, first-principles framework for constructing discrete-time SSMs that is both flexible and modular. Our approach is based on a novel lag operator, which geometrically derives the discrete-time recurrence by measuring how the system's basis functions undergo what we call a \textit{domain expansion} from one timestep to the next. The resulting state matrices are computed via a single inner product involving this operator, enabling a modular design space for creating novel SSMs by flexibly combining different basis functions and time-warping schemes. To validate our framework, we demonstrate that a specific instance exactly recovers the recurrence of the influential HiPPO model. Numerical simulations confirm our derivation, providing new theoretical tools for designing flexible and robust sequence models.

\end{abstract}
\section{INTRODUCTION}
Structured State Space Models (SSMs), which form the backbone of modern architectures like Mamba, have emerged as a powerful architecture for long-range sequence modeling \citep{Gu2023-ut}. Their key innovation is imposing mathematical structure on traditional SSMs to enable effective dynamical memory compression. They achieve performance that rivals Transformers on many tasks while breaking free from the quadratic scaling bottleneck of attention mechanisms \citep{Gu2021-be, Somvanshi2025-or}.

As many real-world phenomena are inherently dynamical systems, SSMs provide a natural choice for learning a latent representation of a signal's history for downstream tasks like prediction or classification. Unlike attention-based models with direct access to the entire input sequence, SSMs operate recurrently and must compress the past into a finite-dimensional state. The efficacy of an SSM, therefore, hinges on the quality of this compression mechanism. In Structured SSMs, this mathematical structure arises from rigorous theorems of continuous-time memory updates, where the SSM's state is structured as coefficients of a signal's history projected onto an orthogonal basis whose domain expands over time. However, the theoretical foundation that enabled this resurgence, most notably the High-order Polynomial Projection Operators (HiPPO) framework \citep{Gu2020-de}, relies on an intricate multi-stage derivation involving continuous-time ODEs and subsequent discretization. This complexity can obscure the geometric intuition behind the state update, making principled extensions difficult.

In this paper, we introduce a novel framework for constructing discrete-time Structured SSMs based on discrete-time geometric principles. The hallmark of our framework is the \textbf{lag operator} ($\sigma_{t} \circ \sigma_{t+1}^{-1}$), a geometric mapping that quantifies the \textit{domain expansion} of the structured projection basis between timesteps. This operator derives the complete discrete recurrence ($\bm{A}_t, \bm{B}_t$) directly from a single inner product. This provides a complete, \textbf{end-to-end discrete-time construction}, as it entirely bypasses the intermediate ODE and discretization steps of traditional methods. A key advantage of this direct construction is its modularity: by defining the time-warping function $\sigma_t$, one can construct varied SSMs with different memory properties, such as models with multi-resolution memory to track long-term trends and fine-grained recent details simultaneously.

To validate our framework, we show theoretically that an exponential warping instance exactly recovers the original HiPPO recurrence (via Propositions 1 and 2), providing a more fundamental geometric foundation for HiPPO itself. Numerical simulations confirm this with (1) precise matrix equivalence to HiPPO counterparts and (2) replication of HiPPO's memory dynamics with our end-to-end recurrence.
\section{CONTINUOUS-TIME ORIGINS OF STRUCTURED SSMS}
The theoretical underpinnings of structured SSMs like Mamba trace back to the Legendre Memory Unit (LMU) \citep{Voelker2019-ir, Voelker2018-be}, which sought to enable continuous memory storage by approximating a continuous-time delay line $u(t-\theta)$ for a given input signal $u(t)$.
Since the delay operator's transfer function in the Laplace domain, $e^{-\theta s}$, is irrational, it cannot be directly realized by a finite-dimensional state-space system. The LMU's key insight was to first approximate this function with a rational Padé approximant, which can then be converted into a state-space form. The resulting dynamical system's state matrices are intrinsically linked to the recurrence relations of Legendre polynomials, establishing them as a natural basis for approximating signal history.

The HiPPO framework \citep{Gu2020-de} generalized this approach by formalizing the online compression of a signal's history into a set of orthogonal polynomial coefficients, effectively solving the problem in reverse, by first approximating the signal and then constructing a dynamic model.

Under the HiPPO framework, the LMU is considered a special case referred to as HiPPO-LegT (Legendre Truncated). A key variant, HiPPO-LegS (Legendre Scaled), uses a time-varying measure $\mu(t,s) = 1/t$ to weight the inner product integral used for projecting the signal onto an orthogonal basis. This measure is defined on the interval $[0,t]$ and ensures that the total weight remains constant as time evolves: the projection coefficients are given by
\begin{equation}
    c_n(t) = \int_0^t u(s) L_n(s/t) \mu(t,s) \, ds,
\end{equation}
where $u(t)$ is the input signal, and $L_n$ is the $n$-th Legendre polynomial (normalized on $[0,1]$). This leads to a continuous-time Ordinary Differential Equation (ODE) for the projection coefficients $\bm{c}(t)$ that is fundamentally time-varying:
\begin{equation}
    \label{eq:hippo-legs}
    \bm{c}'(t) = \frac{1}{t} \bm{A} \bm{c}(t) + \frac{1}{t} \bm{B} u(t),
\end{equation}
where $\bm{c}(t)$ is the state vector of projection coefficients, and $\bm{A}$ and $\bm{B}$ are constant matrices derived from the properties of Legendre polynomials. However, the time-dependent factor $1/t$ makes this system difficult to implement directly as a standard time-invariant convolutional or recurrent model.

A major breakthrough in making these models practical was the Structured State Space Sequence (S4) model \citep{Gu2021-be}, which proposed using a time-invariant version of the system,
\begin{equation}
    \bm{c}'(t) = \bm{A} \bm{c}(t) + \bm{B} u(t).
\end{equation}
This is equivalent to dropping the $1/t$ factor from the HiPPO-LegS ODE (eq.\eqref{eq:hippo-legs}). This simplification proved remarkably effective, achieving state-of-the-art performance on tasks involving long-range dependencies \citep{Gu2021-be}. However, it initially lacked a clear theoretical justification, raising the question of why this approximation worked so well.

The crucial insight connecting these two systems was provided later in \citep{Gu2022-wk}. Corollary C.10 of that work proved that the time-invariant system used by S4 is not an arbitrary approximation, but is itself a principled HiPPO operator under a different construction. Specifically, it is equivalent to projecting the signal onto a basis of Legendre polynomials with an \textit{exponential time warping}, i.e., basis functions of the form $L_n(e^{s-t})$ (e.g., instead of $L_n(s/t)$) with an exponential measure $e^{s-t}$ over the infinite history $(-\infty, t]$: the projection coefficients are given by
\begin{equation}
c_n(t) = \int_{-\infty}^t u(s) L_n(e^{s-t}) e^{s-t} \, ds.
\end{equation}
However, the HiPPO derivation is a multi-stage, intricate process. First, a signal's history is approximated continuously using orthonormal polynomials under a chosen measure. Second, an ODE is derived to govern the evolution of the approximation coefficients, a step that can involve complex calculus with time-varying measures and basis functions. Finally, the resulting continuous-time system must be discretized to yield a practical, discrete-time recurrence. This multi-step pipeline, while elegant and robust, obscures the direct geometric connection between the signal history and the final state update. Our work builds directly on this insight, proposing a general framework where time warping and discretization are not a post-hoc justification, but the central, first-principles mechanism for constructing the SSM recurrence.

\section{OUR FRAMEWORK}

\label{sec:our-framework}
This section details our lag-operator framework, which constructs a discrete-time SSM by consecutively projecting the input signal onto an evolving time-warped basis. We introduce a novel lag operator, $\sigma_{t} \circ \sigma_{t+1}^{-1}$, which directly yields the state transition by quantifying a \textit{domain expansion} of the basis functions from one time step to the next.

\subsection{Definitions}
\label{sec:definitions}
\subsubsection{Orthonormal Basis System}
Consider a canonical interval \(Z\) (e.g., $(0, 1]$ for the shifted orthonormal Legendre polynomials or \([-\pi, \pi]\) for Fourier basis). 
For arbitrary functions $U, V$ supported on $Z$, an inner product is defined as:
\begin{equation}
    \langle U, V \rangle = \int_Z U(z) V(z) \, d\mu(z),
\end{equation}
where $\mu$ is a measure on $Z$ (e.g., the Lebesgue measure $d\mu(z)=dz$ for uniform weighting).
Define an orthonormal basis 
\begin{equation}
\Phi = \{\phi_n(z) \mid n = 0, 1, \dots, N-1, \, z \in Z\},
\end{equation}
with inner product \(\langle \phi_n, \phi_m \rangle = \delta_{nm}\). For normalized Legendre polynomials on $(0,1]$, \(\phi_n(z) = \sqrt{2n+1} P_n(2z-1)\) \citep{Gu2022-wk}, where \(P_n\) are the standard Legendre polynomials.

\subsubsection{Signal}

Consider a discrete-time input sequence \(\{u_k\} = (u_1, u_2, \dots, u_t, \dots) \in \mathbb{R}^\mathbb{N}\), observed sequentially up to time \(t\) (unit timestep; \(t \in \mathbb{N}\)).

To enable continuous-time analysis and basis transformations, interpolate \(u_k\) into a continuous function \(u: \mathbb{R} \to \mathbb{R}\) on time domain \(s\) via zero-order hold (ZOH): for history up to \(t\), 
\begin{equation}    
u(s) = u_k \quad \text{for} \quad s \in [k-1, k), \quad k = 1, \dots, t,
\end{equation}
with \(u(s) = 0\) for \(s < 0\) or \(s > t\). This piecewise-constant representation treats the discrete signal as continuous for projection onto orthonormal bases, preserving sample-and-hold semantics.

\subsubsection{Time Warping}

At time \(t\), the interval of interest is \(T_t = (-\infty, t]\), encompassing the signal history.

Define an invertible map \(\sigma_t\) that \textit{warps} the time interval \(T_t\) to the canonical interval \(Z\) (Fig. \ref{fig:s-z-map}): 
\begin{equation}
    \sigma_t: T_t \rightarrow Z, \quad s \in T_t \mapsto z = \sigma_t(s) \in Z,
\end{equation}
analogous to HiPPO's \(\sigma(t,s)\) \citep{Gu2022-wk}.

Construct an orthonormal basis on \(T_t\) by warping \(\phi_n\):
\begin{equation}
\psi_{t,n}(s) = \phi_n(\sigma_t(s)) = \phi_n \circ \sigma_t(s),
\end{equation}
denoted as \(\Psi_t\):
\begin{equation}
\begin{aligned}
\Psi_t &= \{\psi_{t,n}(s) \mid n = 0,\dots,N-1 \} \\
&= \Phi \circ \sigma_t.
\end{aligned}
\end{equation}


\begin{figure}[H]
    \centering
    \resizebox{0.5\textwidth}{!}{\begin{tikzpicture}[
    font=\sffamily,
    >=Stealth, 
    declare function={
      random_signal(\x) = 0.3*sin(5*\x r) + 0.3*sin(8*\x r + 1) + 0.1*sin(12*\x r + 2) + 0.7;
      scale_range(\x) = \x / (2.3);
      time_warping(\x) = (ln(\x + 1)/2)/0.2 + 0.5;
      weight_constant(\x) = 1;
      weight_exponential(\x) = 2*exp(\x);
    }
]

\begin{scope}[yshift=3cm]
    \draw[black, thick, domain=-4:2.3, samples=401] 
        plot (\x, {random_signal(\x)});
    \draw[ForestGreen, thick, domain=-4:2.3, samples=401] 
        plot (\x, {weight_exponential((\x-2.3))/2});
    \node[black, right, font=\Large] at (2.3, 0.6) {$u(s)$};
    \node[ForestGreen, right, font=\Large] at (2.3, 1.2) {$\omega_t(s)$};
    
    \draw[->] (-4, 0) -- (3, 0) node[right, font=\large] {$s$};
    \node[anchor=west, font=\large] at (-5, 0) {$-\infty$};
    \node[below, font=\large] at (2.3, -0.15) {$t$};
    \draw (2.3, -0.15) -- (2.3, 0.15);
\end{scope}

\begin{scope}[yshift=-1cm]
    \draw[->] (-3, 0) -- (3, 0) node[right, font=\large] {$z$};
    
    \node[below, font=\large] at (2.3, -0.15) {$1$};
    \draw (2.3, -0.15) -- (2.3, 0.15);
    \node[below, font=\large] at (-2.3, -0.15) {$0$};
    \draw (-2.3, -0.15) -- (-2.3, 0.15);

    \tikzset{
    }
    \draw[black, thick, domain=-2.299:2.3, samples=401] 
        plot (\x, {random_signal(time_warping(scale_range(\x)))});
    \draw[ForestGreen, thick, domain=-2.3:2.3, samples=401] 
        plot (\x, {weight_constant(0.4*\x)});
    \node[black, right, font=\Large] at (2.5, 0.6) {$u \circ \sigma^{-1}(z)$};
    \node[ForestGreen, right, font=\Large] at (2.5, 1.2) {$d\mu(z) = dz$};
\end{scope}


\coordinate (map_start_s) at (0.5, 2.6);
\coordinate (map_end_z) at (0.5, 0.2);
\coordinate (map_start_z) at (-0.5, 0.2);
\coordinate (map_end_s) at (-0.5, 2.6);

\node[align=center, font=\Large] at (1.6, 1.5) {$\sigma_t(s)$};
\node[align=center, font=\Large] at (-1.8, 1.5) {$\sigma_t^{-1}(z)$};

\draw[-{Stealth[length=3mm, width=2mm]}] 
    (map_start_s) .. controls (1, 2.0) and (1, 1.0) .. (map_end_z);
    
\draw[-{Stealth[length=3mm, width=2mm]}] 
    (map_start_z) .. controls (-1, 1.0) and (-1.0, 2.0) .. (map_end_s);

\draw[-{Stealth[length=2mm]}, black, dashed] (2.3, 2.3) -- (2.3, -0.6);

\draw[-{Stealth[length=2mm]}, black, dashed] (-4.5, 2.7) .. controls (-4, 1) .. (-2.6, -0.6);

\end{tikzpicture}}
    \caption{Illustration of the time-warping function \(\sigma_t\). 
    The infinite history interval \(T_t = (-\infty, t]\) (s-axis) is mapped to the bounded canonical interval \(Z = (0, 1]\) (z-axis). 
    The input signal \(u(s)\) (black) is weighted by the exponentially decaying measure density \(\omega_t(s) = |\sigma_t'(s)|\) (green curve, shown here for the HiPPO-LegS exponential warp \(\sigma_t(s) = e^{s-t}\)). 
    In the canonical domain the induced measure becomes uniform (\(d\mu(z) = dz\), density = 1, green line). 
    The warped signal \(u \circ \sigma_t^{-1}(z)\) (bottom black curve) shows compression of older history (dense oscillations near \(z=0\)) and higher resolution of recent history (near \(z=1\)).}
    \label{fig:s-z-map}
\end{figure}

\subsubsection{Induced Inner Product and Measure}
The inner product on \(T_t\), denoted \(\langle \cdot, \cdot \rangle_{\omega_{t}}\), is defined on the weighted Lebesgue space \(L_2(T_t, \omega_t ds)\), where \(\omega_t(s) = |\sigma_t'(s)|\) is the induced density ensuring orthonormality. It arises as the \textit{pullback} of the canonical inner product on \(L_2(Z)\) under the map \(\sigma_t^{-1}\).

Let $u(s)$ and $v(s)$ be functions on the time interval $T_t$, defined via composition with the warping map from functions $U(z)$ and $V(z)$ on the canonical interval $Z$,
\begin{equation}
\begin{aligned}
    u(s) &= (U \circ \sigma_t)(s) = U(\sigma_t(s)), \\
    v(s) &= (V \circ \sigma_t)(s) = V(\sigma_t(s)).
\end{aligned}
\end{equation}

Then the inner product on $T_t$ can be derived from $Z$ via change of variables in the inner product.
\begin{equation}
\begin{aligned}
    \langle u, v \rangle_{\omega_{t}} &= \langle U, V \rangle \\
    &= \int_Z U(z) V(z) \, dz \\
    &= \int_{T_t} (U \circ \sigma_t)(s) (V \circ \sigma_t)(s) \, |\sigma_t'(s)| \, ds \\
    &= \int_{T_t} u(s) \, v(s) \, \omega_t(s) \, ds ,
\end{aligned}
\end{equation}
with \(\omega_{t}(s) = |\sigma_t'(s)|\).

This ensures the warped basis \(\Psi_t\) is orthonormal with respect to this inner product:
\begin{equation}
    \langle \psi_{t,n}, \psi_{t,m} \rangle_{\omega_t} = \langle \phi_n, \phi_m \rangle = \delta_{nm}.
\end{equation}

As an example, for exponential forgetting in HiPPO-LegS \citep{Gu2022-wk}, the warping function and inverse are: $z = \sigma_t(s) = e^{s-t}, s = \sigma_t^{-1}(z) = t + \log(z)$, yielding \(\omega_{t}(s) = e^{s-t}\), which integrates to 1 over \(T_t\) (normalizing the ``forgetting rate'').

\subsection{Approximation of the Signal}

Consider approximating a signal \(u\) on \(T_t\) (i.e., the history up to time \(t\)) using the inner product induced by the image $\sigma_t^{-1} L_2(Z)$. Since \(\{\psi_{t,n}\}\) are orthonormal under \(\langle \cdot, \cdot \rangle_{\omega_{t}}\), the optimal \(L^2\) approximation $\hat{u}_t$ is:
\begin{equation}
\begin{aligned}
\label{eq:recon-u-w-projection}
\hat{u}_t(s) &= \sum_{n=0}^{N-1} c_{t,n} \psi_{t,n}(s), \\
c_{t,n} &= \langle \psi_{t,n},\, u \rangle_{\omega_{t}} = \langle \phi_n, \, u \circ \sigma_t^{-1} \rangle.
\end{aligned}
\end{equation}

Equivalently, approximate the warped signal \(U(z) = u(\sigma_t^{-1}(z))\) on \(Z\):
\begin{equation}
\hat{U}_t(z) = \sum_{n=0}^{N-1} c_{t,n} \phi_n(z).
\end{equation}

Note that coefficients \(c_{t,n}\) are equivalent whether computed on \(Z\) or \(T_t\).
In matrix form:
\begin{equation}
\begin{aligned}
    \bm{c}_t &= (c_{t,0}, \dots, c_{t,N-1})^\mathsf{T}, \\
    \bm{\Psi}_t &= (\psi_{t,0}, \dots, \psi_{t,N-1})^\mathsf{T}, \\
    \bm{\Phi} &= (\phi_{0}, \dots, \phi_{N-1})^\mathsf{T}.
\end{aligned}
\end{equation}
Hence:
\begin{equation}
\begin{aligned}
    \hat{u}_{t} &= \bm{\Psi}_t^\mathsf{T} \bm{c}_t, \\
    \hat{U}_{t} &= \bm{\Phi}^\mathsf{T} \bm{c}_t.
\end{aligned}
\end{equation}

Geometrically, this projects the signal's entire history up to $t$ onto an orthogonal basis, compressing it into \(N\) coefficients.

\subsection{Online Updating of the Approximation}
\label{sec:online-approx}
Upon receiving new input $u_{t{+}1}$ at time $t{+}1$, extend the history to $T_{t{+}1} = (-\infty, t{+}1]$ and compute $c_{t+1,n}$.

Define an interim extended signal (Fig.\ref{fig:extended_signal_def}), modeling impulse addition per HiPPO \citep{Gu2020-de}:
\begin{equation}
    \label{eq:extended_signal_def}
    \tilde{u}_{t{+}1}(s) = \hat{u}_{t}(s) + \Input_{t{+}1}(s) \cdot u_{t{+}1},
\end{equation}
where \(\hat{u}_{t}(s)\) is the prior approximation on $T_t = (-\infty, t]$ (zero elsewhere), and \(\Input_{t{+}1}(s)\) is a distribution or function that mixes the new sample linearly over the interval boundary (e.g., Dirac delta\footnote{While the Dirac delta is not in $L_2$, the projection is well-defined distributionally and aligns with HiPPO conventions \citep{Gu2020-de}. See Appendix~\ref{app:b-derivations} for convergence to $L_2$ holds.} 
at $s=t{+}1$, zero-order hold constant over $[t, t{+}1]$, or first-order hold linear interpolation). The theoretical implications of these choices and their convergence to a common continuous-time generator are discussed in Section~\ref{sec:theoretical-recover-hippo} and Appendix~\ref{app:b-discrete-and-continuous-discussion}.

\begin{figure}[H]
    \centering
    \begin{tikzpicture}[
    font=\sffamily,
    >=Stealth, 
    declare function={
      random_signal(\x) = 0.3*sin(5*\x r) + 0.3*sin(8*\x r + 1) + 0.1*sin(12*\x r + 2) + 0.7;
      scale_range(\x) = \x / (2.3);
      time_warping(\x) = (ln(\x + 1)/2)/0.2 + 0.5;
    }
]

\begin{scope}[yshift=3cm]
    \draw[->] (-4, 0) -- (3, 0) node[right] {$s$};
    
    \node[anchor=west] at (-4.8, 0) {$-\infty$};
    \node[below] at (1.0, -0.15) {$t$};
    \draw (1.0, -0.15) -- (1.0, 0.15);

    \draw[black] (1.5, 0) -- (1.5, 0.8); 
    \fill[black] (1.5, 0.8) circle (2pt); 
    \node[black, right] at (1.5, 0.8) {${u}_{t+1}$};
    \node[below] at (1.7, -0.15) {$t+1$};
    \draw (1.5, -0.15) -- (1.5, 0.15);


    \draw[black, thick, domain=-4:1.0, samples=401] 
        plot (\x, {random_signal(\x)});
    \draw[black] (1.0, 0) -- (1.0, 0.65); 
    \node[black, above] at (-0.5, 1.0) {$\hat{u}_t(s)$};
\end{scope}

\end{tikzpicture}
    \caption{Conceptualization of the extended signal $\tilde{u}_{t{+}1}(s)$. The new input $u_{t{+}1}$ is mixed over the interval boundary to the prior signal approximation $\hat{u}_t(s)$. }
    \label{fig:extended_signal_def}
\end{figure}

The updated approximation $\hat{u}_{t{+}1}$ (on $T_{t{+}1}$) is the projection of $\tilde{u}_{t{+}1}$ onto $\bm{\Psi}_{t{+}1}$, yielding the new coefficients $c_{t+1,n}$, we obtain:
\begin{equation}
    \begin{aligned}
        \hat{u}_{t{+}1}(s) 
        &= \sum_{n=0}^{N-1} \langle \psi_{t{+}1,n}, \tilde{u}_{t{+}1} \rangle_{\omega_{t{+}1}} \, \psi_{t{+}1,n}(s) \\
        &= \sum_{n=0}^{N-1} c_{t{+}1,n} \psi_{t{+}1,n}(s). 
    \end{aligned}
\end{equation}

Substituting $\tilde{u}_{t{+}1}$ in eq.\eqref{eq:extended_signal_def}, we obtain:
\begin{equation}
    \begin{aligned}
        c_{t{+}1,n}
        &= \langle \psi_{t{+}1,n}, \,(\hat{u}_{t} + \Input_{t{+}1} \, u_{t{+}1}) \rangle_{\omega_{t{+}1}} \\
        &= \langle \psi_{t{+}1,n}, \hat{u}_{t} \rangle_{\omega_{t{+}1}} 
        + \langle \psi_{t{+}1,n}, \,\Input_{t{+}1} \, u_{t{+}1} \rangle_{\omega_{t{+}1}} \\
        &= \sum_{m=0}^{N-1} c_{t,m} \, \langle \psi_{t{+}1,n},\, \psi_{t,m} \rangle_{\omega_{t{+}1}} \\
        & \quad \quad \quad + \langle \psi_{t{+}1,n}, \, \Input_{t{+}1} \rangle_{\omega_{t{+}1}} \, u_{t{+}1} \\
        &= \sum_{m=0}^{N-1} c_{t,m} (\bm{A}_t)_{nm} + (\bm{B}_t)_{n} \, u_{t{+}1},
    \end{aligned}
\end{equation}

where 
\begin{equation}
\begin{aligned}
\label{eq:define-a-as-lag}
    (\bm{A}_t)_{nm} 
    &= \langle \psi_{t{+}1,n}, \psi_{t,m} \rangle_{\omega_{t{+}1}} \\
    &= \langle \phi_n , \phi_m \circ \sigma_{t} \circ \sigma_{t{+}1}^{-1} \rangle, \\
\end{aligned}
\end{equation}
and
\begin{equation}
\begin{aligned}
\label{eq:define-b}
    (\bm{B}_t)_n &= \langle \psi_{t{+}1,n}, \Input_{t{+}1}(s) \rangle_{\omega_{t{+}1}}.
\end{aligned}
\end{equation}

Here, $\bm{B}_t$ is derived from the projection of Input onto $\bm{\Psi}_{t{+}1}$ (see Appendix~\ref{app:b-discrete-and-continuous-discussion} for the derivation of \(\mathbf{B}_t\) under different input-hold assumptions).
Thus,
\begin{equation}
\begin{aligned}
    \hat{u}_{t{+}1}(s) &= \sum_{n=0}^{N-1} c_{t{+}1,n} \psi_{t{+}1,n}(s) \\
    &= \sum_{n=0}^{N-1} \left( \sum_{m=0}^{N-1} c_{t,m} (\bm{A}_t)_{nm} + (\bm{B}_t)_{n} u_{t{+}1} \right) \psi_{t{+}1,n}(s).
\end{aligned}
\end{equation}

In vector form, this yields the HiPPO SSM recurrence \citep{Gu2020-de}
\begin{equation}
\begin{aligned}
    \hat{u}_{t{+}1} &= \bm{\Psi}_{t{+}1}^\mathsf{T} \bm{c}_{t{+}1}, \\
    \bm{c}_{t{+}1} &= \bm{A}_t \bm{c}_t + \bm{B}_t \, u_{t{+}1}.
\label{eq:ssm-recurrence}
\end{aligned}
\end{equation}

The composition of \(\bm{A}_t\) in eq.\eqref{eq:define-a-as-lag}, which we define as the backward \textbf{lag operator} \( \sigma_{t} \circ \sigma_{t+1}^{-1}\), acts on the canonical basis functions to accommodate the \textit{domain expansion} as time advances from $t$ to $t+1$. This formulation presents a novel insight that geometrically interprets HiPPO's state transition as a ``re-warping'' of projections. A geometric illustration of this domain-expansion process realized by the lag operator is provided in Appendix~\ref{app:domain-expansion} (Figure~\ref{fig:domain-expansion}).

\section{THEORETICAL VALIDATION}

\subsection{Recovering the Time-Invariant HiPPO-LegS System}
\label{sec:theoretical-recover-hippo}
To bridge the discrete-time framework of Section~\ref{sec:online-approx} to continuous-time limits and connect it to the time-invariant HiPPO-LegS model \citep{Gu2022-wk}, we generalize the unit timestep $[t, t+1]$ to an arbitrarily small step $[t, t+\Delta]$. 
We then analyze the limit as $\Delta \to 0$. In this regime, the time-varying matrices $\bm{A}_t$ and $\bm{B}_t$ (from eq.~\eqref{eq:ssm-recurrence}) become step-size-dependent $\bm{A}_\Delta$ and $\bm{B}_\Delta$, which approximate the identity plus a first-order correction: $\bm{A}_\Delta \approx \bm{I} + \Delta \bm{A}_{\text{gen}}$ and $\bm{B}_\Delta \approx \Delta \bm{B}_{\text{gen}}$. These converge to the generators of the continuous-time SSM
\begin{equation}
\begin{aligned}
    \bm{c}'(t) &= \bm{A}_{\text{gen}} \bm{c}(t) + \bm{B}_{\text{gen}} u(t),
\end{aligned}
\end{equation}
via the forward difference: $\bm{c}'(t) = \lim_{\Delta \to 0} \frac{\bm{c}(t+\Delta) - \bm{c}(t)}{\Delta}$. The discrete recurrence $\bm{c}_{t+\Delta} = \bm{A}_\Delta \bm{c}_t + \bm{B}_\Delta u_{t+\Delta}$ thus recovers the continuous dynamics through a first-order Taylor expansion of the lag operator $\sigma_t \circ \sigma_{t+\Delta}^{-1}$ (eq.~\eqref{eq:define-a-as-lag}). For the stable HiPPO-LegS system, the generator $\bm{A}_{\text{gen}}$ relates to the canonical HiPPO matrix via $\bm{A}_{\text{HiPPO}} = -(\bm{A}_{\text{gen}} + \bm{I})^\mathsf{T}$ (incorporating a tilting shift for measure renormalization; see Appendix~\ref{app:a_gen-derivations}), while $\bm{B}_{\text{gen}}$ matches directly. For clarity, we summarize the key matrices and their relationships (see also Table~\ref{tab:matrix-summary} in Appendix~\ref{app:matrix-summary} for a tabular overview):
\begin{itemize}
    \item $\bm{A}_t$ and $\bm{B}_t$: The discrete-time transition and input matrices at time $t$, as defined in eq.\eqref{eq:ssm-recurrence}, for unit timestep ($\Delta=1$).
    \item $\bm{A}_\Delta$ and $\bm{B}_\Delta$: Generalizations of $\bm{A}_t$ and $\bm{B}_t$ for arbitrary timestep $\Delta$, used to analyze the continuous limit. Note that \(\bm{A}_\Delta\) and \(\bm{B}_\Delta\) depend on \(t\) via \(\sigma_t \circ \sigma_{t + \Delta}\), but become time-invariant under stationary warping.
    \item $\bm{A}_{\text{gen}}$ and $\bm{B}_{\text{gen}}$: Continuous-time generators in the ODE $\bm{c}'(t) = \bm{A}_{\text{gen}} \bm{c}(t) + \bm{B}_{\text{gen}} u(t)$, derived as the first-order terms in the Taylor expansion of the discrete recurrence: $\bm{c}(t+\Delta) = \bm{c}(t) + \Delta [\bm{A}_{\text{gen}} \bm{c}(t) + \bm{B}_{\text{gen}} u(t)] + O(\Delta^2)$. Equivalently, $\bm{c}'(t) = \lim_{\Delta \to 0} [\bm{c}(t+\Delta) - \bm{c}(t)] / \Delta$.
    \item $\bm{A}_{\text{HiPPO}}$: The stable state matrix from HiPPO-LegS, related to our $\bm{A}_{\text{gen}}$ by transposition and a tilting shift: $\bm{A}_{\text{HiPPO}} = -(\bm{A}_{\text{gen}}^\mathsf{T} + \bm{I})$.
\end{itemize}

The derivations of $\bm{A}_{\text{gen}}$ and $\bm{B}_{\text{gen}}$ are given in Propositions~\ref{prop:a-gen} and~\ref{prop:b-gen} (with proofs in Appendices~\ref{app:a_gen-derivations} and~\ref{app:b-gen-proof}), and we show their equivalence to HiPPO-LegS below.

The components are:

\textbf{Canonical Interval} $Z$: The interval $(0, 1]$.

\textbf{Canonical Basis} $\bm{\Phi}$: The Legendre polynomials $L_n(z)$ are orthonormal over $(0, 1]$.

\textbf{Warping Function} $\sigma_t$: The exponential map $\sigma_t(s) = e^{s-t}$, which maps the time domain $s \in (-\infty, t]$ to the canonical interval $z \in (0, 1]$.

To derive the continuous-time state generator $\bm{A}_{\text{gen}}$, which governs the evolution of the state in the limit of infinitesimally small time steps, we introduce the following proposition that encapsulates this process for a stationary warping function.

\begin{proposition}[Continuous Generator $\bm{A}_{\text{gen}}$]
\label{prop:a-gen}
For a stationary warping $\sigma_t(s) = f(s - t)$ with inverse $g = f^{-1}$, the generator matrix is
\begin{equation}
(\bm{A}_{\text{gen}})_{nm} = \left\langle \phi_n, \frac{\phi'_m}{g'} \right\rangle = \int_Z \phi_n(z) \frac{\phi'_m(z)}{g'(z)} \, dz.
\end{equation}
(See Appendix~\ref{app:a_gen-derivations} for proof.)
\end{proposition}

Applying Proposition~\ref{prop:a-gen} to the exponential warp ($f(x) = e^x$, $g(z) = \ln(z)$), we obtain
\begin{equation}
    (\bm{A}_{\text{gen}})_{nm} = \int_0^1 L_n(z) \cdot z \cdot L_m'(z) \, dz,
\end{equation}
with diagonals $n$ and off-diagonals $\sqrt{(2n+1)(2m+1)}$ for $n < m$ (upper triangular). This matches the transpose of the reference matrix $\bm{A}^0$ from \citep{Gu2022-wk} (App. C.8), defined as
\begin{equation}
    \label{eq:original-a-0-def}
    (\bm{A}^0)_{nm} =
    \begin{cases} 
        \sqrt{(2n+1)(2m+1)} & \text{if } m < n, \\ 
        n & \text{if } n=m, \\ 
        0 & \text{otherwise}. 
    \end{cases}
\end{equation}
Thus, $\bm{A}_{\text{gen}} = (\bm{A}^0)^\mathsf{T}$. The HiPPO-LegS state matrix is $\bm{A}_{\text{HiPPO}} = -(\bm{A}^0 + \bm{I})$, related by transposition and a shift: $\bm{A}_{\text{HiPPO}} = -(\bm{A}_{\text{gen}} + \bm{I})^\mathsf{T}$, where the shift arises from a ``tilting'' term in the HiPPO derivation, adjusting for measure renormalization.

For the continuous-time input generator $\bm{B}_{\text{gen}}$, our framework is flexible with respect to the input model (e.g., Dirac delta, ZOH, or FOH; see Appendix~\ref{app:b-derivations}). Regardless of the choice, the following proposition shows convergence to the same continuous $\bm{B}_{\text{gen}}$ in the limit $\Delta \to 0$, equivalent to the Dirac boundary evaluation, reassuring that all models recover the HiPPO input matrix.

\begin{proposition}[Continuous Generator for $\bm{B}_{\rm gen}$]
\label{prop:b-gen}
Let the discrete input matrix $\bm{B}_\Delta$ be the coefficient such that the state update includes the term $\bm{B}_\Delta u(t+\Delta)$. 
In the continuous limit $\Delta \to 0$, the generator is given by
\begin{equation}
    (\bm{B}_{\rm gen})_n = \lim_{\Delta \to 0} \frac{(\bm{B}_\Delta)_n}{\Delta} = \phi_n(1) f'(0).
\end{equation}
(See Appendix~\ref{app:b-gen-proof} for proof.)
\end{proposition}

Applying Proposition~\ref{prop:b-gen} to the exponential warp ($f(x) = e^x$, $f'(0)=1$) yields
\begin{equation}
\label{eq:b-gen-def}
(\bm{B}_{\text{gen}})_n = \phi_n(1) = L_n(1) = \sqrt{2n+1},
\end{equation}
which matches the HiPPO-LegS input matrix from \citep{Gu2022-wk} (App. C.8). This confirms consistency with the continuous-time ODE summarized at the start of this section.

Hence, our framework successfully derives generators equivalent to the HiPPO-LegS system, confirming its validity as a novel geometric construction capable of recovering and extending foundational SSMs.

\subsection{Long-Term Behavior of the State Transition and its Connection to the Continuous-Time Generator}
\label{sec:long-term-behavior}
This section validates the continuous-time generator $\bm{A}_{\text{gen}}$ by showing that the discrete-time transition matrix $\bm{A}_t$, derived from a finite time step, converges to a stationary matrix $\bm{A}_\infty$ that is consistent with $\bm{A}_{\text{gen}}$ through the matrix exponential.

For stationary warping $\sigma_t(s) = f(s-t)$ with inverse $g = f^{-1}$, the backward lag operator $\sigma_t \circ \sigma_{t+\Delta}^{-1} (z)= f(\Delta + g(z))$ is time-independent. Thus, $\bm{A}_t$ is constant, and we define the stationary transition matrix $\bm{A}_\infty = \bm{A}_\Delta$ (as summarized in Section~\ref{sec:theoretical-recover-hippo}). Note that \(\bm{A}_\Delta\) depends on \(t\) via \(\sigma_t \circ \sigma^{-1}_{t + \Delta}\) in general, but becomes time-invariant here under stationary warping.
\begin{equation}
(\bm{A}_\infty)_{nm} = (\bm{A}_\Delta)_{nm} = \int_Z \phi_n(z) \phi_m\big(f(\Delta + g(z))\big) \, dz.
\end{equation}

This provides a first-principles justification for time-invariant systems like S4. For HiPPO-LegS exponential warp ($f(x) = e^x$, $g(z) = \ln(z)$),
\begin{equation}
(\bm{A}_\Delta)_{nm} = \int_0^1 L_n(z) L_m(e^\Delta z) \, dz,
\end{equation}
the exact discrete transition over $\Delta$.

The matrix $\bm{A}_\Delta$ relates to $\bm{A}_{\text{gen}}$ by $\bm{A}_\Delta = e^{\Delta \cdot \bm{A}_{\text{gen}}}$. To match the stable HiPPO system, incorporate the tilting term: the stable generator is $\bm{A}_{\text{stable}} = -(\bm{A}_{\text{gen}} + \bm{I})$, and the named $\bm{A}_{\text{HiPPO}} = \bm{A}_{\text{stable}}^{\mathsf{T}} = -(\bm{A}_{\text{gen}}^{\mathsf{T}} + \bm{I})$. The corrected discrete matrix is
\begin{equation}
\label{eq:a-delta-corrected}
\bm{A}_\Delta^{\text{corrected}} = e^{\Delta \cdot \bm{A}_{\text{stable}}} = (\bm{A}_\Delta)^{-1} e^{-\Delta}.
\end{equation}
This matches our numerical experiments and shows the exact discrete transition can be found by inverting $\bm{A}_\Delta$ and applying a scalar decay. The transpose in HiPPO is notational, and our framework derives the underlying stable dynamics.


\section{NUMERICAL VALIDATION}
\label{sec:numerical-validation}
This section provides numerical validation of our lag-operator framework. We confirm our theoretical claims by demonstrating both matrix equivalence and accurate signal reconstruction against the HiPPO-LegS system. While this study serves as a proof of concept, broader benchmarking is deferred to future work.

\subsection{Experimental Setup}
\label{sec:experimental-setup}
We numerically validate our framework by recovering the time-invariant HiPPO-LegS system, using the same components (exponential warp $\sigma_t(s) = e^{s-t}$, orthonormal Legendre basis on $(0,1]$) detailed for the theoretical validation in Section~\ref{sec:theoretical-recover-hippo}. As input, we use a 1D signal from the chaotic Lorenz63 system ($\sigma=10, \rho=28, \beta=8/3$), which provides a challenging, non-stationary signal to test memory compression. For our simulation, we use a basis size of $N=64$, a time step of $\Delta=0.01$ s, and a total duration of $ T=10$ s, parameters chosen to match the original HiPPO demos. \footnote{The original HiPPO source code and demos with parameters can be found at: \url{https://github.com/state-spaces/s4/tree/main/notebooks}}

\subsection{Validation of Core Theoretical Claims}
This subsection compares key matrices and signal reconstructions with those of HiPPO-LegS, ensuring that our approach correctly captures its dynamics.

\subsubsection{Matrix Equivalence Validation}
We assess matrix equivalence by measuring relative difference using the Frobenius norm ($\text{Diff}(\bm{M}_1, \bm{M}_2) = \|\bm{M}_1 - \bm{M}_2\|_F / \|\bm{M}_1\|_F$).

\textbf{Evaluation of $\bm{A}_\infty$ vs. $\exp(\Delta \bm{A}_{\text{gen}})$:} We compare the stationary matrix $\bm{A}_\infty = \bm{A}_\Delta$ (computed via lag operator integration from Section~\ref{sec:long-term-behavior}) with the matrix exponential $\exp(\Delta \bm{A}_{\text{gen}})$, where $\bm{A}_{\text{gen}}$ is from Proposition~\ref{prop:a-gen}. Table~\ref{tab:A_infty_vs_A_gen} shows the differences across various $\Delta$ values, indicating convergence as $\Delta \to 0$.
\begin{table}[h]
    \centering
    \caption{Frobenius relative difference between $\bm{A}_\infty$ and $\exp(\Delta \bm{A}_{\text{gen}})$ across $\Delta$ values}
    \label{tab:A_infty_vs_A_gen}
    \begin{tabular}{lcccc}
        \toprule
        $\mathbf{\Delta}$ & $\mathbf{10^{-4}}$ & $\mathbf{10^{-3}}$ & $\mathbf{10^{-2}}$ & $\mathbf{10^{-1}}$ \\
        \midrule
        Diff & $7.85 \cdot 10^{-12}$ & $3 \cdot 10^{-11}$ & $1.52 \cdot 10^{-9}$ & $4.24 \cdot 10^{-5}$ \\
        \bottomrule
    \end{tabular}
\end{table}

\textbf{Evaluation of $\bm{A}_{\text{gen}}$ vs. $\bm{A}_{\text{HiPPO}}$:} We confirm $\bm{A}_{\text{HiPPO}} = -(\bm{A}_{\text{gen}} + \bm{I})^\mathsf{T}$ by comparing the numerically integrated $\bm{A}_{\text{gen}}$ with the analytical HiPPO-LegS matrix. Table~\ref{tab:a_hippo_vs_a_gen} shows the differences across basis sizes $N$, demonstrating high precision.
\begin{table}[h]
    \centering
    \caption{Frobenius relative difference between $\bm{A}_{\text{HiPPO}}$ and $-(\bm{A}_{\text{gen}} + \bm{I})^\mathsf{T}$ across $N$ values}
    \label{tab:a_hippo_vs_a_gen}
    \begin{tabular}{lccc}
        \toprule
        \textbf{$N$} & \textbf{10} & \textbf{30} & \textbf{50} \\
        \midrule
        Diff & $2.26 \cdot 10^{-14}$ & $3.69 \cdot 10^{-13}$ & $5.26 \cdot 10^{-12}$ \\
        \bottomrule
    \end{tabular}
\end{table}

\textbf{Evaluation of $\bm{A}_\Delta^{\text{corrected}}$ vs. Discretized $\bm{A}_{\text{HiPPO}}$:} We verify that our exact discrete matrix, $\bm{A}_\Delta^{\text{corrected}} = (\bm{A}_\Delta)^{-1} e^{-\Delta}$ aligns with the standard bilinear discretization of $\bm{A}_{\text{HiPPO}}$. Table~\ref{tab:a_correct_vs_a_hippo} shows the differences across $\Delta$ values, indicating convergence as $\Delta \to 0$.
\begin{table}[h]
    \centering
    \caption{Frobenius relative difference between $\bm{A}_\Delta^{\text{corrected}}$ and discretized $\bm{A}_{\text{HiPPO}}$ across $\Delta$ values}
    \label{tab:a_correct_vs_a_hippo}
    \begin{tabular}{lcccc}
        \toprule
        $\mathbf{\Delta}$ & $\mathbf{10^{-4}}$ & $\mathbf{10^{-3}}$ & $\mathbf{10^{-2}}$ & $\mathbf{10^{-1}}$  \\
        \midrule
        Diff & $5.95 \cdot 10^{-5}$ & $0.0257$ & $0.334$ & $0.955$ \\
        \bottomrule
    \end{tabular}
\end{table}

These results show convergence as $\Delta \to 0$ or $N \to \infty$, validating that our lag operator framework accurately recovers HiPPO's continuous and discrete dynamics.

\subsubsection{Geometric Lag Operator Validation}
To illustrate the lag operator's geometric intuition, we apply $\bm{A}_\Delta$ as a backward shift and its inverse as a forward shift on basis functions. For stability correction, the backward shift operator is $\bm{A}_\text{back} = \bm{A}_\Delta e^{\Delta}$ and the forward shift operator is $\bm{A}_\text{for} = \bm{A}_\Delta^{\text{corrected}} = (\bm{A}_\Delta)^{-1} e^{-\Delta}$, such that $\psi_{t \rightarrow t-1, n} = \sum_m (\bm{A}_\text{back})_{nm} \cdot \psi_{t,m}$ and $\psi_{t \rightarrow t+1, n} = \sum_m (\bm{A}_\text{for})_{nm} \cdot \psi_{t,m}$. Figure~\ref{fig:lagged_basis_forward_and_backward} shows both shifts for $n=63$.

The backward shift compresses the basis toward older history, while the forward shift expands it, with amplitude shrinkage reflecting the stabilization correction, consistent with HiPPO's design (noted in Section \ref{sec:theoretical-recover-hippo}).
\begin{figure}[h]
    \centering
    \includegraphics[width=0.9\linewidth]{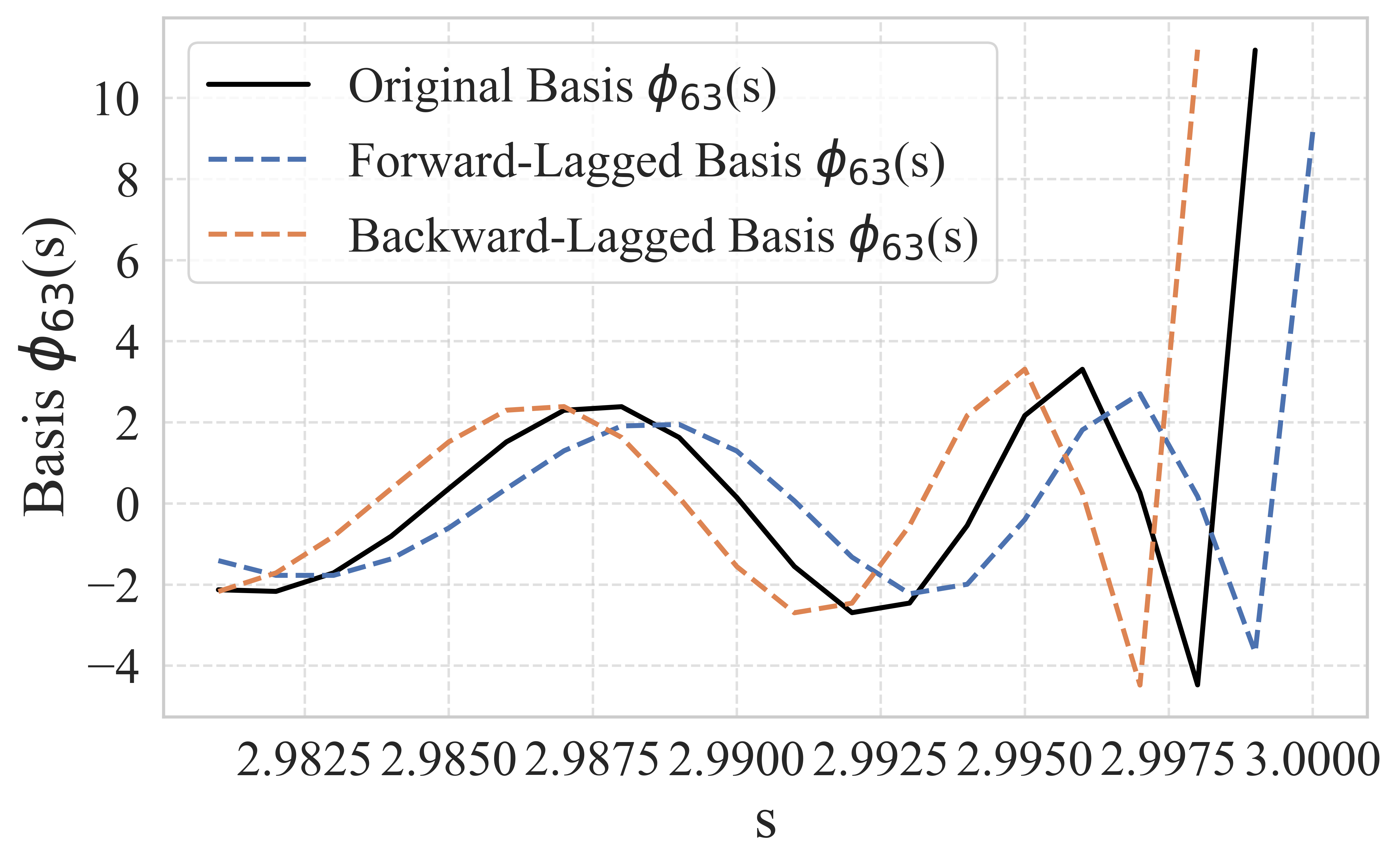}
    \caption{Example basis functions lag-shifted forward and backward using the lag operator $\bm{A}_\Delta$. The backward shift compresses toward old history, while the forward shift expands with amplitude shrinkage due to stabilization. The plot is zoomed near $t = 3.0$ ($\Delta t = 0.001$) to clearly visualize the lag shift; the apparent jaggedness is the expected high-frequency oscillation of high-degree Legendre polynomials at extreme zoom.}
    \label{fig:lagged_basis_forward_and_backward}
\end{figure}

\subsubsection{Evaluating Memory Compression and Forgetting}
We evaluate the quality of our framework's memory compression by visualizing how well the state can reconstruct the signal's history. Crucially, the goal of HiPPO is not to perfectly reconstruct the entire past, but to create a compressed state that optimally represents the \textit{recent} past while gracefully forgetting distant history. The ``reconstructed signal'' shown here is generated by applying the projection defined in eq.~\eqref{eq:recon-u-w-projection} to the final state vector $\bm{c}_T$, which itself is the result of the recurrent update rule. It is therefore a direct visualization of the information stored in this compressed memory.

We compare our Lag Operator HiPPO against the original HiPPO-LegS, using the $x$-component of the chaotic Lorenz63 system as input. Initial conditions were $[1.0, 1.0, 8.0]$; integration used forward Euler with fixed step size $\Delta = 0.01$ (no adaptive tolerances). Only the x-coordinate was used as input $u(t)$. Our reconstruction method uses the derived $\bm{A}_\Delta^{\text{corrected}}$ from eq.\eqref{eq:a-delta-corrected} and the ZOH input vector $\bm{B}_\Delta^{\text{ZOH}}$ from eq.\eqref{eq:define-b-delta-zoh} in Appendix~\ref{app:b-derivations}, while the baseline uses its standard analytical matrices (shown in eq.\eqref{eq:original-a-0-def} and eq.\eqref{eq:b-gen-def}).

Figure~\ref{fig:compare_recon} shows that our model's reconstruction almost perfectly aligns with the original HiPPO-LegS. Both models accurately capture the recent history of the signal (near $t=10$s) while exhibiting the expected exponential decay of information from the distant past (near $t=0$s). The mean squared error (MSE) between the two reconstructions is a negligible $5.76 \cdot 10^{-7}$, confirming that our first-principles framework successfully replicates HiPPO's memory dynamics.
\begin{figure}[h]
    \centering
    \includegraphics[width=\linewidth]{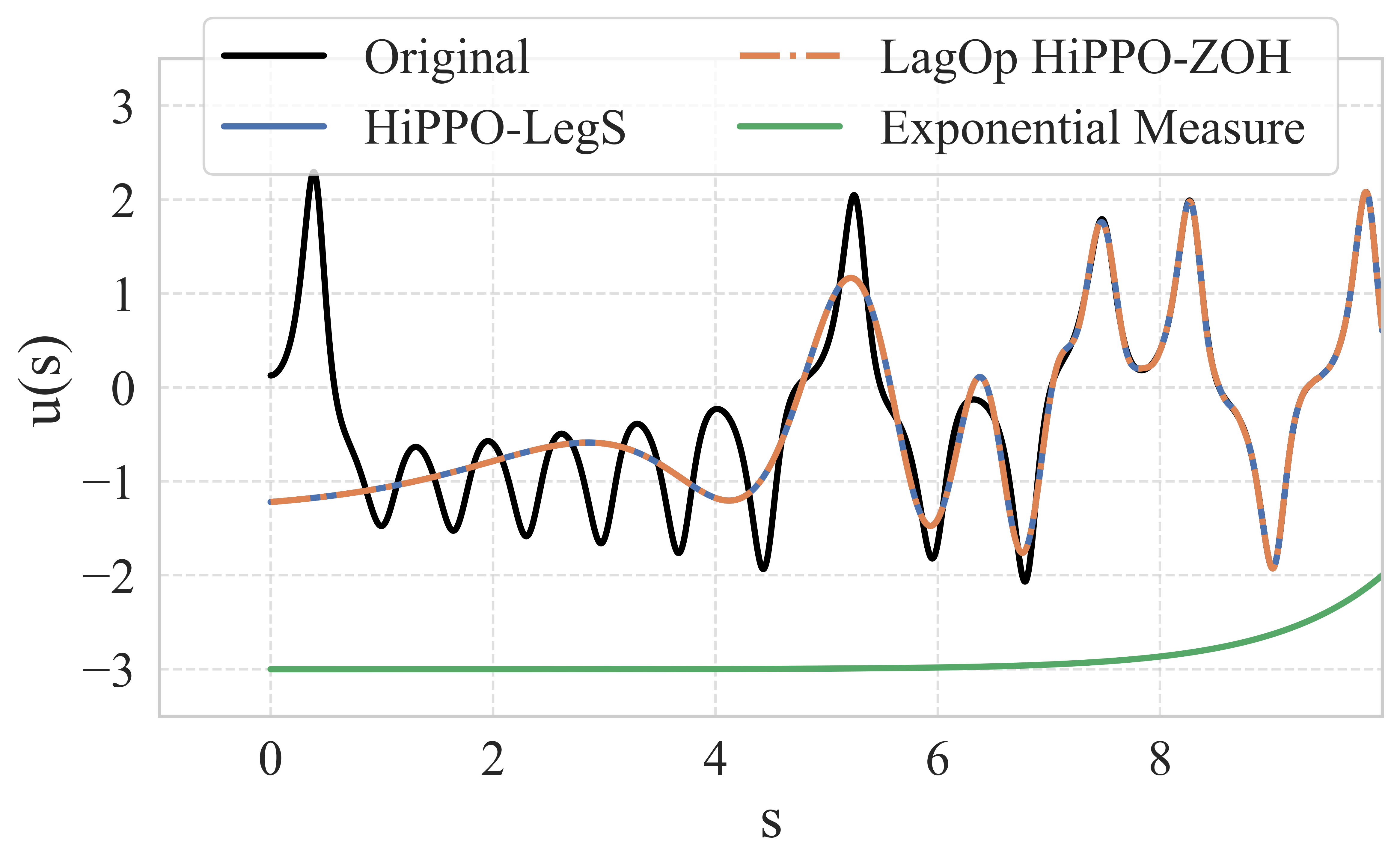}
    \caption{Visualization of the final compressed memory state $\bm{c}_T$ (reconstructed via eq.~\eqref{eq:recon-u-w-projection}) for our Lag Operator model and the HiPPO-LegS baseline. The near-perfect alignment demonstrates that both models accurately capture recent history while forgetting the distant past, a behavior governed by the plotted exponential measure ($\omega_t(s)$), which dictates the memory weighting.}
    \label{fig:compare_recon}
\end{figure}
\section{MODULAR DESIGN: NOVEL SSM  EXAMPLES}

\label{sec:novel-ssm-constructions}

A key advantage of our framework is its modularity. While standard approaches require deriving a new ODE for every change in the measure or basis, our geometric construction allows us to independently select a warping function $\sigma_t$ and a canonical basis $\bm{\Phi}$. The Lag Operator $\sigma_t \circ \sigma_{t+\Delta}^{-1}$ then automatically derives the consistent discrete-time recurrence. In this section, we outline two novel SSM designs enabled by our framework. While we defer the full implementation and downstream benchmarking of these models to future work, we provide these theoretical sketches to demonstrate the flexibility of decoupling the time-warp from the basis.

\subsection{Multi-Resolution Lag-SSM (Fast/Slow Memory)}
Standard HiPPO models enforce a single memory decay rate across the entire state. However, many applications require tracking both long-term trends and fine-grained recent details. Our framework enables the construction of a unified state space with heterogeneous decay rates.

\textbf{Construction:} Construct a hybrid basis set $\bm{\Phi}_{\text{hybrid}}$ on the canonical interval $Z$ that combines ``Fast'' components (e.g., associated with a rapid decay warp $\sigma_t^{(f)}(s) = e^{2(s-t)}$) and ``Slow'' components (e.g., $\sigma^{(s)}_t(s) = e^{0.5(s-t)}$).

\textbf{Geometric Insight:} Rather than running two independent SSMs, we can enforce orthogonality across this entire mixed set using Gram--Schmidt orthogonalization within the canonical domain $Z$. This ensures the state variables are uncorrelated and optimally span the history space. The Lag Operator then acts on this orthogonalized hybrid basis to derive the correct recurrence matrix $\bm{A}_t$. This allows the model to track history at multiple timescales simultaneously within a single coherent state representation.

\subsection{Frequency-Preserving Lag-SSM (Unwarped Oscillations)}
\label{sec:frequency-preserving-lag-ssm-short}
A major limitation of the standard exponential warp is that it ``stretches'' the basis, turning constant frequencies in the $z$-domain into ``chirps'' (variable frequency) in the $s$-domain. This is suboptimal for signals with sustained oscillatory components (e.g., audio, periodic biomedical signals).

\textbf{Construction:} Define $\sigma_t$ as a piecewise-linear map that partitions history (the $s$-domain) into windows of fixed length $L$. Each window maps linearly to a segment of $Z$ with a decaying length $h_n > 0$ (where $\sum h_n = 1$). For hybrid capability, we augment a Fourier basis on $Z$ (for oscillations) with a Legendre subset (for transients), applying Gram--Schmidt orthogonalization. Full mathematical details appear in Appendix~\ref{app:frequency-preserving-details}.

\textbf{Geometric Insight:} This induces a piecewise-constant measure $\omega_t(s) = h_n / L$ on each window. Crucially, this maps the Fourier basis to pure, unwarped sinusoids in the time domain $s$ while preserving orthogonality. The lag operator directly yields the discrete recurrence, bypassing complex ODE derivations despite the discontinuous measure. The result is an SSM implementing a ``sliding-window Fourier transform'' with infinite decaying memory, optionally hybridized for multi-scale dynamics.
\section{DISCUSSION AND FUTURE WORK}

This work introduces a novel, geometrically grounded framework for constructing discrete-time SSMs based on inner-product projections and a unique lag operator. This approach offers several advantages over the established multi-stage paradigm:

\begin{itemize}[nosep]
    \item \textbf{Conceptual Clarity and Unified Memory Control:} Our framework provides a clear, geometric foundation for SSMs. It replaces the complex, multi-stage derivations typical in the HiPPO formulation with a direct projection onto warped bases, where the state transition is intuitively understood as a ``re-warping'' of the signal's compressed history. Crucially, this geometric view unifies memory control into a single component: the time-warp $\sigma_t$. The system's memory profile is dictated by the induced measure $\omega_t = |\sigma'_t|$, which is a direct derivative of the warp. This offers a more straightforward approach than the original HiPPO formulation, which requires a separate ``tilting'' term for stabilization and measure renormalization.

    \item \textbf{A Modular and Interpretable Design Space:} This unified view of memory leads to a powerful and modular design space. The state dynamics are determined by two independent choices: the canonical basis $\bm{\Phi}$ and the time-warping function $\sigma_t$. This decoupling is powerful because, as established, the choice of $\sigma_t$ is a direct and interpretable way to engineer the system's memory via its induced measure. This enables sophisticated designs such as \textbf{multi-resolution memory}, in which distinct memory profiles are assigned to different subsets of a basis. For example, low-order basis functions could be paired with a slow-decaying warp to model long-term trends, while high-order functions use a fast-decaying warp to capture fine-grained, recent details. This allows a single SSM to operate on multiple timescales simultaneously. Furthermore, the ability to make $\sigma_t$ time-dependent provides a natural path toward adaptive, non-stationary dynamics.
\end{itemize}

Our proof-of-concept validation provides a solid foundation for the framework. Numerical results show precise matrix equivalence with HiPPO-LegS, and a low MSE ($5.76 \cdot 10^{-7}$) between the reconstructed memory states of our model and the baseline. This validates that our framework correctly replicates the core function of the HiPPO recurrence, which is to maintain a compressed representation of history, rather than to achieve perfect reconstruction of the entire past.

In practice, this compressed memory state serves as the foundation for downstream components, such as a prediction or classification layer. Furthermore, the framework extends naturally to multivariate inputs by broadcasting the transition matrix $A_t$ across channels (or inserting a mixing layer on inputs/outputs), exactly as in S4 and Mamba. The framework is also robust to the precise discretization of the input; all standard hold assumptions converge to the same continuous-time generator (Appendix~\ref{app:b-discrete-and-continuous-discussion}). By demonstrating that our first-principles framework can produce a memory component that behaves identically to the original, we validate its use as an interchangeable building block. This paves the way for exploring the novel memory designs our framework enables (e.g., the multi-resolution models mentioned above) in future work on downstream tasks. Future research should build on this foundation in several key areas:

\begin{itemize}[nosep]
    \item \textbf{Analysis of Discretization Schemes:} Our experiments successfully validated the framework using a Zero-Order Hold (ZOH) input matrix. An exciting avenue for future work is to systematically explore other input assumptions, such as First-Order Hold (FOH) (see Appendix~\ref{app:b-derivations}), and analyze their resulting stability and scaling properties. This would lead to a deeper understanding of how different signal interpolation methods interact with the system's memory.
    
    \item \textbf{Benchmarking on Downstream Tasks:} Now that the memory component is validated, future efforts should focus on comprehensive benchmarking. This involves integrating the flexible SSMs derived from our framework into larger models and evaluating their performance on standard downstream tasks, such as long-range sequence classification and time-series forecasting, to quantify the benefits of novel memory designs.
\end{itemize}
This paper serves as a foundational step, presenting a novel theoretical perspective that clarifies underlying mechanisms in state-space models and paves the way for more flexible and robust sequence modeling.

\section*{Acknowledgements}
S.~Tomonaga and K.~Doya were supported by the COI-NEXT grant from the Japan Science and Technology Agency (JST). 
N.~Murata was supported by JSPS KAKENHI Grant Numbers JP23H05492 and JP23K24909.

\bibliography{paperpile}

@ARTICLE{Gu2020-de,
  title     = "{HiPPO: Recurrent memory with optimal polynomial projections}",
  author    = "Gu, Albert and Dao, Tri and Ermon, Stefano and Rudra, A and Ré, C",
  editor    = "Larochelle, H and Ranzato, M and Hadsell, R and Balcan, M F and
               Lin, H",
  journal   = "Neural Information Processing Systems",
  publisher = "Curran Associates, Inc.",
  volume    = "abs/2008.07669",
  pages     = "1474--1487",
  month     =  "17~" # aug,
  year      =  2020,
  eprint    = "2008.07669"
}

@ARTICLE{Gu2023-ut,
  title         = "{Mamba: Linear-time sequence modeling with selective state
                   spaces}",
  author        = "Gu, Albert and Dao, Tri",
  journal       = "arXiv [cs.LG]",
  month         =  "1~" # dec,
  year          =  2023,
  archivePrefix = "arXiv",
  primaryClass  = "cs.LG",
  eprint        = "2312.00752"
}

@ARTICLE{Gu2021-be,
  title  = "{Efficiently Modeling Long Sequences with Structured State Spaces}",
  author = "Gu, Albert and Goel, Karan and Re, Christopher",
  month  =  "6~" # oct,
  year   =  2021
}

@ARTICLE{Voelker2018-be,
  title       = "{Improving Spiking Dynamical Networks: Accurate Delays,
                 Higher-Order Synapses, and Time Cells}",
  author      = "Voelker, Aaron R and Eliasmith, Chris",
  affiliation = "Centre for Theoretical Neuroscience and David R. Cheriton
                 School of Computer Science, University of Waterloo, Waterloo,
                 ON N2L 3G1, Canada arvoelke@uwaterloo.ca. Centre for
                 Theoretical Neuroscience, University of Waterloo, Waterloo, ON
                 N2L 3G1, Canada celiasmith@uwaterloo.ca.",
  journal     = "Neural computation",
  volume      =  30,
  number      =  3,
  pages       = "569--609",
  month       =  mar,
  year        =  2018,
  pmid        =  29220306,
  issn        = "0899-7667,1530-888X",
  language    = "en"
}

@ARTICLE{Gu2022-wk,
  title         = "{How to Train Your HiPPO: State Space Models with Generalized
                   Orthogonal Basis Projections}",
  author        = "Gu, Albert and Johnson, Isys and Timalsina, Aman and Rudra,
                   Atri and Ré, Christopher",
  journal       = "arXiv [cs.LG]",
  month         =  "24~" # jun,
  year          =  2022,
  archivePrefix = "arXiv",
  primaryClass  = "cs.LG",
  eprint        = "2206.12037"
}

@ARTICLE{Voelker2019-ir,
  title    = "{Legendre Memory Units: Continuous-Time Representation in
              Recurrent Neural Networks}",
  author   = "Voelker, Aaron and Kajić, Ivana and Eliasmith, Chris",
  journal  = "Advances in Neural Information Processing Systems",
  volume   =  32,
  year     =  2019
}

@ARTICLE{Somvanshi2025-or,
  title         = "{From S4 to Mamba: A comprehensive survey on Structured State
                   Space Models}",
  author        = "Somvanshi, Shriyank and Islam, Md Monzurul and Mimi, Mahmuda
                   Sultana and Polock, Sazzad Bin Bashar and Chhetri, Gaurab and
                   Das, Subasish",
  journal       = "arXiv [cs.LG]",
  month         =  "21~" # mar,
  year          =  2025,
  archivePrefix = "arXiv",
  primaryClass  = "cs.LG",
  eprint        = "2503.18970"
}

\section*{Checklist}
\begin{enumerate}

  \item For all models and algorithms presented, check if you include:
  \begin{enumerate}
    \item A clear description of the mathematical setting, assumptions, algorithm, and/or model. 
    
    [Yes] Detailed in Section~\ref{sec:our-framework} and Appendix~\ref{app:b-discrete-and-continuous-discussion}.
    
    \item An analysis of the properties and complexity (time, space, sample size) of any algorithm. 
    
    [Yes] An analysis of the framework's mathematical properties is provided theoretically in Section~\ref{sec:theoretical-recover-hippo} and numerically in Section~\ref{sec:numerical-validation}. A formal computational complexity analysis is outside the scope of this paper.

    \item (Optional) Anonymized source code, with specification of all dependencies, including external libraries. 
    
    [Yes]
  \end{enumerate}

  \item For any theoretical claim, check if you include:
  \begin{enumerate}
    \item Statements of the full set of assumptions of all theoretical results. 
    
    [Yes] Stated in Section~\ref{sec:our-framework} and Section~\ref{sec:theoretical-recover-hippo}.
    
    \item Complete proofs of all theoretical results.
    
    [Yes] Details in Appendix~\ref{app:a_gen-derivations} and Appendix~\ref{app:b-discrete-and-continuous-discussion}.
    
    \item Clear explanations of any assumptions. 
    
    [Yes] Assumptions for the proposed framework are stated in Section~\ref{sec:our-framework}. 
  \end{enumerate}

  \item For all figures and tables that present empirical results, check if you include:
  \begin{enumerate}
    \item The code, data, and instructions needed to reproduce the main experimental results (either in the supplemental material or as a URL). 
    
    [Yes] We provide supplementary code.
    
    \item All the training details (e.g., data splits, hyperparameters, how they were chosen). 

    [Yes] Stated in Section~\ref{sec:experimental-setup}.

    \item A clear definition of the specific measure or statistics and error bars (e.g., with respect to the random seed after running experiments multiple times). 
    
    [Yes] We report Mean Squared Error (MSE) for signal reconstruction. As our experiment is a deterministic validation of a theoretical derivation, results are consistent across runs, making error bars not applicable for this analysis.

    \item A description of the computing infrastructure used. (e.g., type of GPUs, internal cluster, or cloud provider). 
    
    [Not Applicable] The numerical validation involves standard matrix operations that are not computationally intensive and can be reproduced on any modern personal computer.
  
  \end{enumerate}

  \item If you are using existing assets (e.g., code, data, models) or curating/releasing new assets, check if you include:
  \begin{enumerate}
    \item Citations of the creator If your work uses existing assets. 
    
    [Yes] Mentioned in Section~\ref{sec:experimental-setup}
    
    \item The license information of the assets, if applicable. 
    
    [Yes] We cite the original papers and provide a footnote in Section~\ref{sec:experimental-setup} linking to the public source code repository for the assets used.
    
    \item New assets either in the supplemental material or as a URL, if applicable. 
    
    [Yes] Code provided in supplemental material.

    \item Information about consent from data providers/curators. 

    [Not Applicable] We use an artificial dataset. (see Section~\ref{sec:experimental-setup}.
    
    \item Discussion of sensible content if applicable, e.g., personally identifiable information or offensive content. 

    [Not Applicable] We use an artificial dataset from a dynamical model. (see Section~\ref{sec:experimental-setup}.)
  \end{enumerate}

  \item If you used crowdsourcing or conducted research with human subjects, check if you include:
  \begin{enumerate}
    \item The full text of instructions given to participants and screenshots. 
    
    [Not Applicable] Did not use crowdsourcing nor research human subjects.

    \item Descriptions of potential participant risks, with links to Institutional Review Board (IRB) approvals if applicable. 
    
    [Not Applicable] Did not use crowdsourcing nor research human subjects.

    \item The estimated hourly wage paid to participants and the total amount spent on participant compensation. 

    [Not Applicable] Did not use crowdsourcing nor research human subjects.
  \end{enumerate}

\end{enumerate}

\newpage

\onecolumn
\appendix
\aistatstitle{Appendix}


\section{NOTATION SUMMARY FOR DISCRETE-TO-CONTINUOUS MATRICES}
\label{app:matrix-summary}

\begin{table}[H]
\centering
\caption{Summary of key matrices in the discrete-to-continuous limit.}
\begin{tabular}{l p{6cm} l}
    \toprule
    Matrix & Description & Relation \\
    \midrule
    $\bm{A}_t$, $\bm{B}_t$ & Time-varying discrete transition and input (unit step, eq.~\eqref{eq:ssm-recurrence}) & General case; $\bm{A}_t = \langle \bm{\Phi}, \bm{\Phi} \circ \sigma_t \circ \sigma_{t+1}^{-1} \rangle$ \\
    $\bm{A}_\Delta$, $\bm{B}_\Delta$ & Step-size-dependent discrete (arbitrary $\Delta$) & $\bm{A}_\Delta \approx \bm{I} + \Delta \bm{A}_{\text{gen}}$; $\bm{B}_\Delta \approx \Delta \bm{B}_{\text{gen}}$ \\
    $\bm{A}_{\text{gen}}$ & Continuous state generator & $\bm{A}_{\text{HiPPO}} = -(\bm{A}_{\text{gen}} + \bm{I})^\mathsf{T}$ (tilted/transposed) \\
    $\bm{B}_{\text{gen}}$ & Continuous input generator & Matches HiPPO directly: $(\bm{B}_{\text{gen}})_n = \phi_n(1) f'(0)$ \\
    $\bm{A}_{\text{HiPPO}}$ & Stable HiPPO-LegS matrix (from~\citep{Gu2020-de}) & Derived as special case of $\bm{A}_{\text{gen}}$ \\
    \bottomrule
\end{tabular}
\label{tab:matrix-summary}
\end{table}

\section{PROOF OF PROPOSITION~\ref{prop:a-gen}: CONTINUOUS GENERATOR $\bm{A}_{\text{gen}}$}
\label{app:a_gen-derivations}
This derivation assumes the stationary warping and continuous limit as summarized in Section~\ref{sec:theoretical-recover-hippo}.

\begin{proposition}[Proposition~\ref{prop:a-gen} Restated: Continuous Generator $\bm{A}_{\text{gen}}$]
\label{prop:a-gen-restated}
For a stationary warping $\sigma_t(s) = f(s - t)$ with inverse $g = f^{-1}$, the generator matrix is
\begin{equation}
(\bm{A}_{\text{gen}})_{nm} = \left\langle \phi_n, \frac{\phi'_m}{g'} \right\rangle = \int_Z \phi_n(z) \frac{\phi'_m(z)}{g'(z)} \, dz.
\end{equation}
\end{proposition}

The derivation follows the backward-lag convention from eq.\eqref{eq:define-a-as-lag} with $\sigma_t \circ \sigma_{t+\Delta}^{-1}$.

First, find the inverse $\sigma_t^{-1}$ by solving $\sigma_t(s) = f(s-t)$: $s = t + g(z)$, so $\sigma_t^{-1}(z) = t + g(z)$.

Next, compute $\sigma_t \circ \sigma_{t{+}\Delta}^{-1}$:
\begin{equation}
    \sigma_t \circ \sigma_{t+\Delta}^{-1} = f(\Delta + g(z)).
\end{equation}
To obtain $\bm{A}_t \approx \bm{I} + \Delta \bm{A}_{\text{gen}}$, evaluate the first-order Taylor expansion of $\phi_m(\sigma_t \circ \sigma_{t+\Delta}^{-1}(z))$ at $\Delta = 0$:
\begin{equation*}
    \begin{aligned}
        \text{Generator}(\phi_m)
        &= \left. \frac{d}{d\Delta} \phi_m \Big( f \big( \Delta + g(z) \big) \Big) \right|_{\Delta = 0} \\
        &= \left. \phi'_m \Big( f \big( \Delta + g(z) \big) \Big)
            f' \big( \Delta + g(z) \big) \right|_{\Delta = 0}.
    \end{aligned}
\end{equation*}
At $\Delta = 0$, using $f(g(z)) = z$, this is $\phi'_m(z) \cdot f'(g(z))$. Thus,
\begin{equation}
    (\bm{A}_{\text{gen}})_{nm} = \langle \phi_n, \phi'_m \cdot (f' \circ g) \rangle = \Big\langle \phi_n, \frac{\phi'_m}{g'} \Big\rangle.
\end{equation}
\qed

\section{DERIVATION OF INPUT VECTOR $\bm{B}$ AND PROOF OF PROPOSITION~\ref{prop:b-gen}}
\label{app:b-discrete-and-continuous-discussion}
\subsection{The Input Vector $\bm{B}$ and Different Hold Assumptions}
\label{app:b-derivations}
Here, we derive $\bm{B}_\Delta$ and its limit to $\bm{B}_{\text{gen}}$ under different hold assumptions, consistent with the continuous ODE $\bm{c}'(t) = \bm{A}_{\text{gen}} \bm{c}(t) + \bm{B}_{\text{gen}} u(t)$ (see Section~\ref{sec:theoretical-recover-hippo}).

In the main text, the new input is modeled generally, leading to the recurrence $\bm{c}_{t{+}1} = \bm{A}_t \bm{c}_t + \bm{B}_t u_{t{+}1}$. The specific form of the discrete input matrix $\bm{B}_\Delta$ depends on the assumption made about the signal's behavior over the interval $[t, t+\Delta]$. Here, we derive $\bm{B}_\Delta$ for three common models under a stationary warp $\sigma_t(s) = f(s-t)$, where $f$ is increasing, $f(0)=1$, and $g=f^{-1}$.

The $n$-th component of the input matrix $(\bm{B}_\Delta)_n$ represents the full contribution of the input to the state update, computed as the projection of the input shape, $\text{Input}(s; u_{t+\Delta})$, onto the basis at time $t{+}\Delta$:
\begin{equation}
\begin{aligned}
    (\bm{B}_\Delta)_n 
    &= \langle \psi_{t+\Delta,n}, \text{Input}(s; u_{t+\Delta}) \rangle_{\omega_{t+\Delta}} \\
    &= \int_{-\infty}^{t+\Delta} \psi_{t+\Delta,n}(s) \cdot \text{Input}(s; u_{t+\Delta}) \cdot \omega_{t+\Delta}(s) \, ds.
\end{aligned}
\end{equation}
For stationary warping $\sigma_{t+\Delta}(s) = f(s - t - \Delta)$ with the inverse $g = f^{-1}$, substituting $s - t = \Delta + g(z)$ into the input and performing a change of variables, the inner product simplifies to:
\begin{equation}
    (\bm{B}_\Delta)_n = \int_{f(-\Delta)}^{1} \phi_n(z) \cdot \text{Input}(t + \Delta + g(z); u_{t+\Delta}) \, dz,
\end{equation}
assuming $f$ maps the interval appropriately (limits adjusted for the hold).

Note that for Dirac and ZOH, $\text{Input}(s; u_{t+\Delta})$ is defined as a u-independent shape scaled by $u_{t+\Delta}$, so the full contribution is $\bm{B}_\Delta u_{t+\Delta}$; for FOH, the input shape embeds the linear interpolation of $u_t$ and $u_{t+\Delta}$, which we factor to isolate the matrix component.

\subsubsection{Dirac Delta Input (Impulse)}
This model assumes the input is an instantaneous impulse at the boundary $s = t+\Delta$, simplifying analysis but approximating discrete semantics in the limit.
\begin{equation}
    \text{Input}(s; u_{t+\Delta}) = \delta(s - (t+\Delta)).
\end{equation}
The projection simplifies to a point evaluation of the warped basis function:
\begin{equation}
\begin{aligned}
    (\bm{B}^\delta_\Delta)_n 
    &= \langle \psi_{t+\Delta,n}, \text{Input}(s; u_{t+\Delta}) \rangle_{\omega_{t+\Delta}} \\
    &= \int_{-\infty}^{t+\Delta} \psi_{t+\Delta,n}(s) \, \delta(s - (t+\Delta)) \, \omega_{t+\Delta}(s) \, ds \\
    &= \psi_{t+\Delta,n}(t+\Delta) \, \omega_{t+\Delta}(t+\Delta) \\
    &= \phi_n(\sigma_{t+\Delta}(t+\Delta)) \, \omega_{t+\Delta}(t+\Delta) \\
    &= \phi_n(1) \, |f'(0)|,
\end{aligned}
\end{equation}
assuming $f$ is increasing so $|f'(0)| = f'(0)$. Since this is a boundary evaluation, $(\bm{B}^\delta_\Delta)_n$ is constant with respect to $\Delta$. The full input contribution is then $(\bm{B}^\delta_\Delta)_n u_{t+\Delta}$.

\subsubsection{Zero-Order Hold (ZOH) Input}
\label{app:zoh-def}
This model assumes the input is held constant over the interval, representing a piecewise-constant signal, which more accurately captures discrete-time semantics.
\begin{equation}
    \text{Input}(s; u_{t+\Delta}) = \bm{1}_{[t,t+\Delta]}(s).
\end{equation}
Applying the change of variables $z = \sigma_{t+\Delta}(s)$, the integration limits become $[f(-\Delta), 1]$, yielding:
\begin{equation}
    \label{eq:define-b-delta-zoh}
    (\bm{B}^{\text{ZOH}}_\Delta)_n = \int_{f(-\Delta)}^1 \phi_n(z) \, dz.
\end{equation}
High-order basis functions ($\phi_n(z)$ for large $n$) oscillate near boundaries, so this integral averages behavior over the interval, scaled by $u_{t+\Delta}$.

\subsubsection{First-Order Hold (FOH) Input}
\label{app:foh-def}
This model assumes the input is interpolated linearly from $u_t$ to $u_{t+\Delta}$. The input shape is a linear ramp. To isolate the matrix component (linear in $u_{t+\Delta}$ as per main text), we factor accordingly:
\begin{equation}
    \text{Input}(s; u_{t+\Delta}) = \bm{1}_{[t,t+\Delta]}(s) \cdot \left[ u_t + \frac{u_{t+\Delta} - u_t}{\Delta} (s-t) \right].
\end{equation}
After the change of variables where $s - t = \Delta + g(z)$, the projection becomes:
\begin{equation}
\begin{aligned}
    (\bm{B}^{\text{FOH}}_\Delta)_n 
    &= \int_{f(-\Delta)}^{1} \phi_n(z) \left[ u_t + \frac{u_{t+\Delta} - u_t}{\Delta} (\Delta + g(z)) \right] dz \\
    &= \int_{f(-\Delta)}^{1} \phi_n(z) \left[ u_{t+\Delta} + \frac{u_{t+\Delta} - u_t}{\Delta} g(z) \right] dz \\
    &= u_{t+\Delta} I_1(\Delta) + (u_{t+\Delta} - u_t) \frac{I_g(\Delta)}{\Delta},
\end{aligned}
\end{equation}
where 
\begin{equation}
\begin{aligned}
    I_1(\Delta) &= \int_{f(-\Delta)}^{1} \phi_n(z) \, dz, \\
    I_g(\Delta) &= \int_{f(-\Delta)}^{1} \phi_n(z) g(z) \, dz.
\end{aligned}
\end{equation}
The $u_t$ term contributes to the higher-order correction, which vanishes in the limit (see below).

\subsection{Proof of Proposition~\ref{prop:b-gen}: Convergence to the Continuous Generator $\bm{B}_{\text{gen}}$}
\label{app:b-gen-proof}
While the discrete matrices differ, we now show that the ZOH and FOH models converge to the same continuous generator $\bm{B}_{\rm gen}$, consistent with the Dirac case. This demonstrates that for small $\Delta$, the system's dynamics are insensitive to the hold order.

To obtain the generator, assume $u(s)$ is a smooth continuous function, so $u_t = u(t)$ and $u_{t+\Delta} = u(t + \Delta)$.

In the continuous limit, the SSM ODE is $\bm{c}'(t) = \bm{A}_{\rm gen} \bm{c}(t) + \bm{B}_{\rm gen} u(t)$, so the discrete approximation is $\bm{c}_{t+\Delta} \approx \bm{c}_t + \Delta [\bm{A}_{\rm gen} \bm{c}_t + \bm{B}_{\rm gen} u(t)] + O(\Delta^2)$. Thus, the full input contribution over $\Delta$ should be $\bm{B}_\Delta u_{t+\Delta} \approx \Delta \bm{B}_{\rm gen} u(t)$, and
\begin{equation}
    \bm{B}_{\rm gen} u(t) = \lim_{\Delta \to 0} \frac{\bm{B}_\Delta u_{t+\Delta}}{\Delta}.
\end{equation}
Since $\bm{B}_\Delta$ is the coefficient multiplying $u(t+\Delta)$, we isolate $\bm{B}_{\rm gen}$ directly by
\begin{equation}
    (\bm{B}_{\rm gen})_n = \lim_{\Delta \to 0} \frac{(\bm{B}_\Delta)_n}{\Delta}.
\end{equation}

\begin{proposition}[Proposition~\ref{prop:b-gen} Restated: Continuous Input Generator]
\label{prop:b-gen-restated}
Let the discrete input matrix $\bm{B}_\Delta$ be the coefficient such that the state update includes the term $\bm{B}_\Delta u(t+\Delta)$. In the continuous limit, the generator $\bm{B}_{\rm gen}$ is given by:
\begin{equation}
    (\bm{B}_{\rm gen})_n = \lim_{\Delta \to 0} \frac{(\bm{B}_\Delta)_n}{\Delta} = \phi_n(1) f'(0).
\end{equation}
\end{proposition}

\begin{proof}
The proof relies on analyzing the first-order behavior of the ZOH and FOH integrals as $\Delta \to 0$. Recall from Section~\ref{app:b-derivations} (ZOH and FOH definitions) that 
    \begin{equation*}
    \begin{aligned}
        I_1(\Delta) &= \int_{f(-\Delta)}^{1} \phi_n(z) \, dz, \\
        I_g(\Delta) &= \int_{f(-\Delta)}^{1} \phi_n(z) g(z) \, dz.
    \end{aligned}
    \end{equation*}
    
For \textbf{ZOH} (see eq.~\eqref{eq:define-b-delta-zoh} in Section~\ref{app:zoh-def}), the matrix is $(\bm{B}^{\text{ZOH}}_\Delta)_n = I_1(\Delta)$, and the full contribution is $I_1(\Delta) u_{t+\Delta}$.

First, we find the first-order approximation of these integrals. Using the Leibniz integral rule on $I_1(\Delta)$:
\begin{equation}
    \frac{d}{d\Delta} I_1(\Delta) = -\phi_n(f(-\Delta)) \cdot f'(-\Delta) \cdot (-1).
\end{equation}
Evaluating at $\Delta=0$ (and using $f(0)=1, g(1)=0$):
\begin{equation}
    \left. \frac{d}{d\Delta} I_1(\Delta) \right|_{\Delta=0} = \phi_n(1) f'(0).
\end{equation}
This implies the first-order Taylor expansion is $I_1(\Delta) = \Delta \cdot \phi_n(1)f'(0) + O(\Delta^2)$. \footnote{We note that while the Leibniz rule provides a direct path to the boundary derivative, we frame the analysis via first-order Taylor expansion here to emphasize the intuitive linear limit consistent with the main text's generator derivation.}

Similarly for $I_g(\Delta)$:
\begin{equation}
    \left. \frac{d}{d\Delta} I_g(\Delta) \right|_{\Delta=0} = \phi_n(1) g(1) f'(0) = \phi_n(1) \cdot 0 \cdot f'(0) = 0.
\end{equation}
This implies $I_g(\Delta) = O(\Delta^2)$, so it vanishes faster than first order.

Now, we apply these limits. For \textbf{ZOH}:
\begin{equation}
    \lim_{\Delta \to 0} \frac{(\bm{B}^{\text{ZOH}}_\Delta)_n u_{t+\Delta}}{\Delta} = \lim_{\Delta \to 0} \frac{I_1(\Delta) u_{t+\Delta}}{\Delta} \approx \lim_{\Delta \to 0} \frac{\Delta \cdot \phi_n(1)f'(0) u(t)}{\Delta} = \phi_n(1) f'(0) u(t),
\end{equation}
using $u_{t+\Delta} \to u(t)$ by continuity. Thus $(\bm{B}_{\rm gen})_n = \phi_n(1) f'(0)$.

For \textbf{FOH} (see Section~\ref{app:foh-def} for the FOH derivation), since $u$ is continuous, $u_{t+\Delta} \to u(t)$ and $u_{t+\Delta} - u_t \to 0$ as $\Delta \to 0$. The full projection is:
\begin{equation}
(\bm{B}^{\text{FOH}}_\Delta)_n = u_{t+\Delta} I_1(\Delta) + (u_{t+\Delta} - u_t) \frac{I_g(\Delta)}{\Delta}.
\end{equation}
Dividing by $\Delta$ and taking the limit:
\begin{equation}
\begin{aligned}
    \lim_{\Delta \to 0} \frac{(\bm{B}^{\text{FOH}}_\Delta)_n}{\Delta} &= \lim_{\Delta \to 0} \left[ u_{t+\Delta} \frac{I_1(\Delta)}{\Delta} + (u_{t+\Delta} - u_t) \frac{I_g(\Delta)}{\Delta^2} \right] \\
  &= u(t)[\phi_n(1)f'(0)] + \left( \lim_{\Delta \to 0} (u_{t+\Delta} - u_t) \right) \cdot \left( \lim_{\Delta \to 0} \frac{I_g(\Delta)}{\Delta^2} \right) \\
  &= u(t)[\phi_n(1)f'(0)] + 0 \cdot \text{const.} \\
  &= \phi_n(1)f'(0) u(t),
\end{aligned}
\end{equation}
Thus $(\bm{B}_{\rm gen})_n = \phi_n(1) f'(0)$, matching ZOH and the Dirac boundary case.

The consistency arises because the $u_t$ term in FOH contributes only to higher-order terms (e.g., via $I_g(\Delta)$), which vanish in the $\Delta \to 0$ limit, leaving the $u_{t+\Delta} \approx u(t)$ term dominant.
\end{proof}

\subsection{Conclusion for HiPPO-LegS}
For the HiPPO-LegS special case ($f(x)=e^x$, $f'(0)=1$, and $\phi_n(z)$ being normalized Legendre polynomials on $[0,1]$), all three input models are consistent in the continuous limit, yielding the well-known generator:
\begin{equation}
    (\bm{B}_{\text{gen}})_n = \phi_n(1) \cdot 1 = \sqrt{2n+1},
\end{equation}
where the full contribution is $(\bm{B}_{\text{gen}})_n u(t) = \sqrt{2n+1} u(t)$.

\section{GEOMETRIC ILLUSTRATION OF THE LAG OPERATOR AND \textit{DOMAIN EXPANSION}}
\label{app:domain-expansion}

The lag operator \(\sigma_t \circ \sigma_{t+1}^{-1}\) geometrically encodes the \emph{domain expansion} of the structured projection basis when time advances from \(t\) to \(t+1\). Figure~\ref{fig:domain-expansion} visualizes this process for a generic warping function \(\sigma_t\).

\begin{figure}[h]
    \centering
    \begin{tikzpicture}[
    font=\sffamily,
    >=Stealth, 
    declare function={
      random_signal(\x) = 0.3*sin(5*\x r) + 0.3*sin(8*\x r + 1) + 0.1*sin(12*\x r + 2) + 0.7;
      scale_range(\x) = \x / (1.5);
      time_warping(\x) = (ln(\x + 1)/2)/0.2 + 0.5;
      time_warping_inverse(\x) = 2*exp(\x - 1) - 1;
      weight_constant(\x) = 0.5;
      weight_exponential(\x) = 2*exp(\x) - 1;
      L_0(\x) = 1 * 0.5 + 0.5;
      L_1(\x) = \x  * 0.5 + 0.5;
      L_2(\x) = (3*(\x)^2 - 1)/2  * 0.5 + 0.5;
      L_3(\x) = (5*(\x)^3 - 3*(\x))/2) * 0.5 + 0.5;
    }
]

\begin{scope}[yshift=2cm, xshift=0cm]
    \draw[black] (1.5, 0) -- (1.5, 0.8); 
    \fill[black] (1.5, 0.8) circle (2pt); 
    \node[black, above, font=\large] at (1.5, 0.8) {${u}_{t+1}$};
    \node[black, right, font=\Large] at (2., 0.6) {$\tilde{u}_{t+1}(s)$};

    \draw[black, thick, domain=-4:1.0, samples=401] 
        plot (\x, {random_signal(\x)});
    \draw[black] (1.0, 0) -- (1.0, 0.65); 
    \node[black, above, font=\large] at (-0.5, 1.0) {$\hat{u}_t(s)$};
   
    \draw[->] (-4, 0) -- (2, 0) node[right, font=\large] {$s$};
    \node[anchor=west, font=\large] at (-5, 0) {$-\infty$};
\end{scope}

\begin{scope}[yshift=0cm, xshift=0cm]
    \draw[->] (-4, 0) -- (2, 0) node[right, font=\large] {$s$};
    \node[anchor=west, font=\large] at (-5, 0) {$-\infty$};
    \node[below, font=\large] at (1., -0.15) {$t$};
    \draw (1., -0.15) -- (1., 0.15);
    \draw[dashed] (1, 2) -- (1, 0);
    
    \draw[black, thick, domain=-4:1, samples=401, dashed] 
        plot (\x, {L_0(time_warping_inverse(\x))});
    \draw[black, thick, domain=-4:1, samples=401, dashed] 
        plot (\x, {L_1(time_warping_inverse(\x))});
    \draw[black, thick, domain=-4:1, samples=401, dashed] 
        plot (\x, {L_2(time_warping_inverse(\x))});
    \draw[black, thick, domain=-4:1, samples=401, dashed] 
        plot (\x, {L_3(time_warping_inverse(\x))});
    \node[black, right, font=\Large] at (1.7, 1.2) {$\psi_{t,n}(s)$};
\end{scope}

\begin{scope}[yshift=0cm, xshift=0cm]
    \draw[->] (-4, 0) -- (2, 0) node[right, font=\large] {$s$};
    \node[anchor=west, font=\large] at (-5, 0) {$-\infty$};
    
    \node[below, font=\large] at (1.8, -0.15) {$t+1$};
    \draw (1.5, -0.15) -- (1.5, 0.15);
    \draw[dashed] (1.5, 2) -- (1.5, 0);
    
    \draw[red, thick, domain=-4:1.5, samples=401] 
        plot (\x, {L_0(time_warping_inverse(scale_range(\x)))});
    \draw[red, thick, domain=-4:1.5, samples=401] 
        plot (\x, {L_1(time_warping_inverse(scale_range(\x)))});
    \draw[red, thick, domain=-4:1.5, samples=401] 
        plot (\x, {L_2(time_warping_inverse(scale_range(\x)))});
    \draw[red, thick, domain=-4:1.5, samples=401] 
        plot (\x, {L_3(time_warping_inverse(scale_range(\x)))});
    \node[red, right, font=\Large] at (2., 0.6) {$\psi_{t{+}1,n}(s)$};
\end{scope}

\begin{scope}[yshift=-2cm, xshift=-0.5cm]
    \draw[->] (-2, 0) -- (2, 0) node[right, font=\large] {$z$};
    \node[below, font=\large] at (1.5, -0.15) {$1$};
    \draw (1.5, -0.15) -- (1.5, 0.15);
    \node[below, font=\large] at (-1.5, -0.15) {$0$};
    \draw (-1.5, -0.15) -- (-1.5, 0.15);
    \tikzset{
    }
    \draw[black, thick, domain=-1.5:1.5, samples=401] 
        plot (\x, {L_0(scale_range(\x)))});
    \draw[black, thick, domain=-1.5:1.5, samples=401] 
        plot (\x, {L_1(scale_range(\x)))});
    \draw[black, thick, domain=-1.5:1.5, samples=401] 
        plot (\x, {L_2(scale_range(\x)))});
    \draw[black, thick, domain=-1.5:1.5, samples=401] 
        plot (\x, {L_3(scale_range(\x)))});
    \node[black, right, font=\Large] at (1.5, 0.6) {$\phi_n(z)$};
\end{scope}

\begin{scope}[yshift=-2cm, xshift=0cm]
    \draw[->, thick, black, dashed] (-2.5, 1.8) .. controls (-2.8, 1.) .. (-2.1, 0.3);
    \node[align=center, font=\large] at (-3.6, 0.8) {$z=\sigma_t(s)$};
    \draw[<-, red, thick] (-2., 1.8) .. controls (-2.5, 1.) .. (-2., 0.4);
    \node[align=center, red, font=\large] at (-1., 1.4) {$s=\sigma^{-1}_{t{+}1}(z)$};
\end{scope}
\end{tikzpicture}
    \caption{Illustration of the domain-expansion process realized by the lag operator \(\sigma_t \circ \sigma_{t+1}^{-1}\). 
    \textbf{Top:} The prior approximation \(\hat{u}_t(s)\) (black) on \(T_t = (-\infty, t]\) is extended by the new input sample \(u_{t+1}\) (black dot) to form the interim signal \(\tilde{u}_{t+1}(s)\) on \(T_{t+1} = (-\infty, t+1]\). 
    \textbf{Middle:} The old warped basis functions \(\psi_{t,n}(s)\) (dashed black) are re-warped into the new basis \(\psi_{t+1,n}(s)\) (solid red) to accommodate the expanded domain. 
    \textbf{Bottom:} All basis functions are images of the \emph{same} canonical orthonormal basis \(\phi_n(z)\) on \(Z = [0,1]\) under the respective warping maps \(\sigma_t\) and \(\sigma_{t+1}\). 
    The arrows highlight the geometric action of the lag operator, which is exactly the composition that yields the state-transition matrix \(A_t\) via a single inner product (eq.~\eqref{eq:define-a-as-lag}).}
    \label{fig:domain-expansion}
\end{figure}

This visualization illustrates the geometric action of the lag operator, specifically the domain expansion and compression of the basis functions over a time step. The state transition matrix $\bm{A}_\Delta$ is simply the algebraic projection of this geometric shift, demonstrating how the system's dynamics are constructed directly from time-warping principles rather than continuous-time ODEs.
\section{DETAILS ON THE FREQUENCY-PRESERVING LAG-SSM}
\label{app:frequency-preserving-details}

As briefly sketched in Section~\ref{sec:frequency-preserving-lag-ssm-short}, a major limitation of the standard exponential warp is that it stretches signals in time. A basis function that oscillates at a fixed frequency in the canonical domain $Z$ becomes a chirp (variable frequency) in the time domain $s$. This makes it inefficient for signals with sustained oscillatory components. 

Our framework addresses this by designing a time-warp $\sigma_t$ that preserves frequency in the $s$-domain while still providing infinite decaying memory. Figure~\ref{fig:s-z-map-periodic} visualizes the piecewise-linear frequency-preserving time warp of a fixed-frequency oscillatory basis function in the time domain.

\subsection{Construction of the Piecewise-Linear Warp}
The Frequency-Preserving Lag-SSM is built as a \emph{hybrid} model. It combines two families of bases and associated warps: one subset designed to preserve frequency of oscillation in the $s$-domain ($\in T_t$), and another subset (e.g., Legendre polynomials with the standard exponential warp) to capture transients and decaying memory. We elaborate on this construction below.

\subsection{Geometric Consequence: Unwarped Fourier Basis}
\label{app:frequency-preserving-derivation}
The motivation is that even when the canonical basis $\Phi$ on $Z$ contains oscillatory components, the standard exponential warp stretches them into chirps (variable frequency) in the $s$-domain. The piecewise-linear warp introduced below resolves this issue while preserving both orthogonality and infinite decaying memory, as well as the periodicity of the basis functions in the $s$-domain.

Recall the two domains: the physical time domain $s \in T_t$ and the canonical domain $z \in Z$. The inner product on $Z$ uses the Lebesgue measure $d\mu(z) = dz$:
\begin{equation}
    \langle U, V \rangle = \int_Z U(z) V(z) \, d\mu(z).
\end{equation}

First, for the $s$-domain measure $\omega_t$ on $(-\infty, t] = T_t$, we consider a step function with infinitely many (decaying) $n$ steps ($n=1, 2, ...$) of length $L$ and the values are constant across the interval $(t-nL,\, t-(n-1)L]$ of height $h_n$ where $h_n>0$ and $\sum\limits^{\infty}_{n=1}h_n=1$. For example, $h_n = C\,e^{-n}$ or $\frac{C}{n^2}$ ($C$ is for normalization).

Then we define the time-warp $\sigma_t(s)$ by dividing $Z$ into intervals of length $h_n$.

More precisely, for any $s$ in the interval $(t-nL, t-(n-1)L]$, we map it to the $Z$'s divided interval $H_n$, such that
\begin{equation}
H_n = \left(1-\sum\limits^n_{i=1}h_i, \,1-\sum\limits^{n-1}_{i=1}h_i\right).
\end{equation}

By decomposing $s$ to see its lower bound and mid-point via the length and ratio,
\begin{equation}
s = (t-nL) +  \left(\frac{s-t+nL}{L}\right) \, L.
\end{equation}
We can get the $z$ in the corresponding interval $H_n$,
\begin{equation}
\sigma_t = z = \left(1-\sum\limits^n_{i=1}h_i\right) + \left(\frac{s-t+nL}{L}\right) \, h_n .
\end{equation}

Now $\sigma_t$ is a decaying piecewise-linear time-warp yielding a constant measure $\omega_t(s) = |\sigma_t'(s)| = h_n$ over the interval $H_n$. This means that if we define a Fourier basis in the $s$-domain with period $L$, i.e.\ $\{\sin(2\pi k/L), \cos(2\pi k/L); k=0,1,2,\dots\}$, the waveform is piecewise preserved in the $z$-domain within each interval $H_n$, ensuring that orthogonality is maintained in the $z$-space (Figure~\ref{fig:s-z-map-periodic}).

\begin{figure}[h!]
    \centering
    \begin{tikzpicture}[
    font=\sffamily,
    >=Stealth, 
    declare function={
          sin_wave(\x) = 0.5*sin(4*\x r)+1;
          stepwise_decay(\x) = 
            ifthenelse(\x > -pi/2 && \x <= 0, 1,
            ifthenelse(\x > -2*pi/2 && \x <= -pi/2, 1/2,
            ifthenelse(\x > -3*pi/2 && \x <= -2*pi/2, 1/4, 
            ifthenelse(\x >= -4*pi/2 && \x <= -3*pi/2, 1/8, 0))));
        weight_constant(\x) = 1;
    }
]

\begin{scope}[yshift=3cm, xshift=2.3cm]
    \draw[black, thick, domain=-2*pi:0, samples=401] 
        plot (\x, {sin_wave(\x)+1});
        
    \draw[ForestGreen, thick, domain=-2*pi:0, samples=401] 
        plot (\x, {stepwise_decay(\x)});

    \draw[ForestGreen, dashed] (-1/2*pi, 1) -- (-1/2*pi, 0);
    \draw[ForestGreen, dashed] (-pi, 1) -- (-pi, 0);
    \draw[ForestGreen, dashed] (-3/2*pi, 1) -- (-3/2*pi, 0);
        
    \node[black, right, font=\Large] at (0.2, 2.0) {$\psi(s)$};
    \node[ForestGreen, right, font=\Large] at (0.2, 1.0) {$\omega_t(s)$};
    
    \draw[->] (-2*pi, 0) -- (1.5, 0) node[right, font=\large] {$s$};
    \node[anchor=west, font=\large] at (-2*pi-1.5, 0) {$-\infty$};
    \node[below, font=\large] at (0, -0.15) {$t$};
    \draw (0, -0.15) -- (0, 0.15);
\end{scope}

\begin{scope}[yshift=-1cm, xshift=2.3cm]
    \draw[->] (-5.0, 0) -- (1.5, 0) node[right, font=\large] {$z$};
    
    \node[below, font=\large] at (0, -0.15) {$1$};
    \draw (0, -0.15) -- (0, 0.15);
    \node[below, font=\large] at (-4.6, -0.15) {$0$};
    \draw (-4.6, -0.15) -- (-4.6, 0.15);

    \draw[black, thick, domain=-2.4:0, samples=100] 
        plot (\x, {0.5*sin(360 * \x / 2.4) + 1});
        
    \draw[black, thick, domain=-3.6:-2.4, samples=100] 
        plot (\x, {0.5*sin(360 * (\x - (-2.4)) / 1.2) + 1});
        
    \draw[black, thick, domain=-4.2:-3.6, samples=100] 
        plot (\x, {0.5*sin(360 * (\x - (-3.6)) / 0.6) + 1});
        
    \draw[black, thick, domain=-4.5:-4.2, samples=100] 
        plot (\x, {0.5*sin(360 * (\x - (-4.2)) / 0.3) + 1});
        

    \draw[ForestGreen, thick, domain=-4.6:0, samples=20] 
        plot (\x, {weight_constant(\x)});

    \draw[ForestGreen, dashed] (-2.4, 1) -- (-2.4, 0);
    \draw[ForestGreen, dashed] (-3.6, 1) -- (-3.6, 0);
    \draw[ForestGreen, dashed] (-4.2, 1) -- (-4.2, 0);

    \node[black, right, font=\Large] at (0.2, 0.6) {$\psi \circ \sigma_t^{-1}(z)$};
    \node[ForestGreen, right, font=\Large] at (0.2, 1.3) {$d\mu(z) = dz$};
\end{scope}


\coordinate (map_start_s) at (0.5, 2.7);
\coordinate (map_end_z)   at (0.5, 0.3);
\coordinate (map_start_z) at (-0.5, 0.3);
\coordinate (map_end_s)   at (-0.5, 2.7);

\node[align=center, font=\Large] at (1.6, 1.5) {$\sigma_t(s)$};
\node[align=center, font=\Large] at (-1.8, 1.5) {$\sigma_t^{-1}(z)$};

\draw[-{Stealth[length=3mm, width=2mm]}] 
    (map_start_s) .. controls (1, 2.0) and (1, 1.0) .. (map_end_z);
    
\draw[-{Stealth[length=3mm, width=2mm]}] 
    (map_start_z) .. controls (-1, 1.0) and (-1.0, 2.0) .. (map_end_s);

\begin{scope}[gray, dashed, -{Stealth[length=2mm]}]
    \draw (2.3, 2.3) -- (2.3, -0.8);
    
    
    \draw (-4.8, 2.6) .. controls (-4, 1) .. (-2.3, -0.8);
\end{scope}

\end{tikzpicture}
   \caption{Illustration of the piecewise-linear time warping for the Frequency-Preserving Lag-SSM (Sec.~\ref{sec:frequency-preserving-lag-ssm-short}). The infinite history interval \(T_t = (-\infty, t]\) (top, \(s\)-axis) is partitioned into fixed-length windows of size \(L\), each weighted by a decaying constant \(\omega_n\) (green step function \(\omega_t(s) = \omega_n/L\) on the \(n\)-th window). The invertible map \(\sigma_t\) linearly compresses each window into a segment of the canonical interval \(Z = [0,1]\) whose length is proportional to \(\omega_n\). Consequently, the induced measure on \(Z\) remains uniform (\(d\mu(z) = dz\), green line), so a Fourier basis defined on \(Z\) maps back to pure, unwarped sinusoids in the original time domain (visible in the warped signal \(u \circ \sigma_t^{-1}(z)\)). This yields an SSM that implements a sliding-window Fourier transform with infinite decaying memory, derived directly via the lag operator without any ODE.}
    \label{fig:s-z-map-periodic}
\end{figure}

\subsection{Extension to Hybrid Model: Transient Decaying Memory}
To capture the remaining degrees of freedom (e.g., smooth transients and decaying memory), we can augment the Fourier basis with a Legendre subset under the standard exponential warp. The full hybrid basis is then orthogonalized with respect to the oscillatory basis via the Gram--Schmidt process.

\end{document}